%% file: neurips_2026.tex
\documentclass{article}
\PassOptionsToPackage{numbers}{natbib}
\usepackage[preprint]{neurips_2026}

\usepackage{microtype}
\usepackage{hyperref}
\usepackage{url}
\usepackage{booktabs}
\usepackage{listings}
\usepackage{mdframed}
\usepackage{amsmath}
\usepackage{amssymb}
\usepackage{amsthm}
\input{macros}

\usepackage{lineno}

\input{extra_setup.tex}
\definecolor{darkblue}{rgb}{0, 0, 0.5}
\hypersetup{colorlinks=true, citecolor=darkblue, linkcolor=darkblue, urlcolor=darkblue}

\title{Fine-Tuning Without Forgetting via Loss-Adaptive Learning Rates}

\author{Parjanya Prajakta Prashant$^{*,\dagger}$ \quad Jiongli Zhu$^{*}$ \quad Aldan Creo \quad Babak Salimi \\
University of California San Diego \\
$^{*}$Equal contribution \\
$^{\dagger}$Corresponding author: \texttt{pprashant@ucsd.edu}
}

\begin{document}

\maketitle

\begin{abstract}

Fine-tuning large language models on new data improves task performance but degrades capabilities learned during pretraining, a phenomenon known as catastrophic forgetting. Existing methods mitigate this by modifying the fine-tuning objective to suppress high-loss tokens or sequences, but these tokens are essential for learning new tasks, especially those with poor pretraining coverage. In such settings, hard tokens should still contribute to learning, so forgetting must be controlled without suppressing them. We identify a simple mechanism for doing so: per-step forgetting is bounded by the product of the learning rate and the square root of the current training loss. This suggests that high-loss batches are especially prone to inducing forgetting. Motivated by this observation, we introduce \ourmethod, a loss-adaptive learning-rate schedule that reduces the learning rate on high-loss batches and increases it as the model converges, while leaving the fine-tuning objective unchanged. Across knowledge acquisition, science, and low-resource language adaptation benchmarks, \ourmethod\ reduces forgetting by 93\% on average while matching the task performance of standard fine-tuning. On Qwen3-4B knowledge acquisition, \ourmethod\ cuts TruthfulQA degradation by $5\times$ and reverses HaluEval degradation, while better preserving confidence calibration. Overall, our results show that learning-rate schedules are an effective tool to shape model behavior during fine-tuning, beyond just target-task optimization. Code is available at \url{https://github.com/parjanya20/forgetting-lr-schedule}.

\end{abstract}

\input{intro-new.tex}
\input{related_work}
\input{method-new}

\input{experiments_preserve}

\input{experiments_hallucination}

\input{conclusion}

\section*{Acknowledgments}
We gratefully acknowledge support from the Modal Academic Compute Grant.

\bibliographystyle{plain}
\bibliography{ref}

\input{appendix}

\end{document}

%% file: macros.tex
\usepackage{enumitem}
\usepackage{pgfplots}
\usepackage{multirow}
\usepackage[most]{tcolorbox}
\usepgfplotslibrary{colorbrewer}
\usepackage{pgfplotstable}
\usepackage{wrapfig}
\usepgfplotslibrary{groupplots}
\usepackage{subcaption}
\pgfplotsset{compat=1.18}
\newtheorem{theorem}{Theorem}
\newtheorem{lemma}{Lemma}
\newtheorem{corollary}{Corollary}

\newtheorem{definition}{Definition}
\newtheorem{assumption}{Assumption}

\usepackage{xcolor}
\usepackage{colortbl}
\definecolor{cOurs}{RGB}{228,26,28}

\newcommand{\ourmethod}{\text{FINCH}}

\newcommand{\KL}{\mathrm{KL}}
\newcommand{\CE}{\mathrm{CE}}

%% file: intro-new.tex
\section{Introduction}
\label{sec:introduction}
Large language models (LLMs) are increasingly specialized through fine-tuning on narrow, task-specific corpora~\citep{singhal2025toward, roziere2023code, mecklenburg2024injecting, chen2024monolingual}.
While such adaptation can substantially improve target-task performance, it often degrades broader capabilities acquired during pretraining, a phenomenon known as \textit{catastrophic forgetting}~\citep{mccloskey1989catastrophic, french1999catastrophic}. Such degradation leads to increased hallucination, weakened safety alignment, and degraded instruction-following and reasoning~\citep{qi2023fine, jiang2024refine, luo2025empirical, kalajdzievski2024scaling, gekhman2024does}.
A common mitigation is to replay pretraining data during fine-tuning~\citep{rolnick2019experience, de2019episodic, scialom2022fine}, but this is rarely possible in practice because the pretraining data for most modern LLMs is partially or fully proprietary~\citep{grattafiori2024llama, yang2025qwen3, li2024revisiting}.

This has motivated work in the \textit{data-oblivious} setting, where only the task corpus is available~\citep{sanyal2025upweighting}.
A common observation across these methods is that tokens or sequences assigned high loss by the pretrained model are the primary drivers of forgetting.
This motivates suppressing their influence during fine-tuning: Sanyal et al.~\cite{sanyal2025upweighting} downweight high-loss sequences, Lin et al.~\cite{lin2025sft} reweight losses at the token level, and Wu et al.~\cite{wu2025mitigating} mask tokens whose entropy exceeds a threshold.
However, these approaches struggle on tasks where high-loss tokens are essential to the new target capability. For example, knowledge acquisition requires learning new names and facts the model has never seen, while low-resource language adaptation requires learning new vocabulary and grammar. Suppressing hard tokens in these settings hurts new-task performance without reliably reducing forgetting, and existing methods often fail to match vanilla SFT on new-task accuracy in this regime (Figure~\ref{fig:intro-right}).

\definecolor{plotbg}{RGB}{255,252,240}
\definecolor{cStandardSFT}{RGB}{215,48,39}
\definecolor{cFLOW}{RGB}{252,141,89}
\definecolor{cDFT}{RGB}{145,91,175}
\definecolor{cTALR}{RGB}{254,224,144}
\definecolor{cSTM}{RGB}{69,117,180}
\definecolor{cL2}{RGB}{116,173,209}
\definecolor{cWiSE}{RGB}{255,127,0}
\definecolor{cLoRA}{RGB}{77,175,74}
\definecolor{cOurs}{RGB}{228,26,28}
\definecolor{plotbggrey}{RGB}{245,245,245}
\definecolor{plotbgyellow}{RGB}{255,252,240}

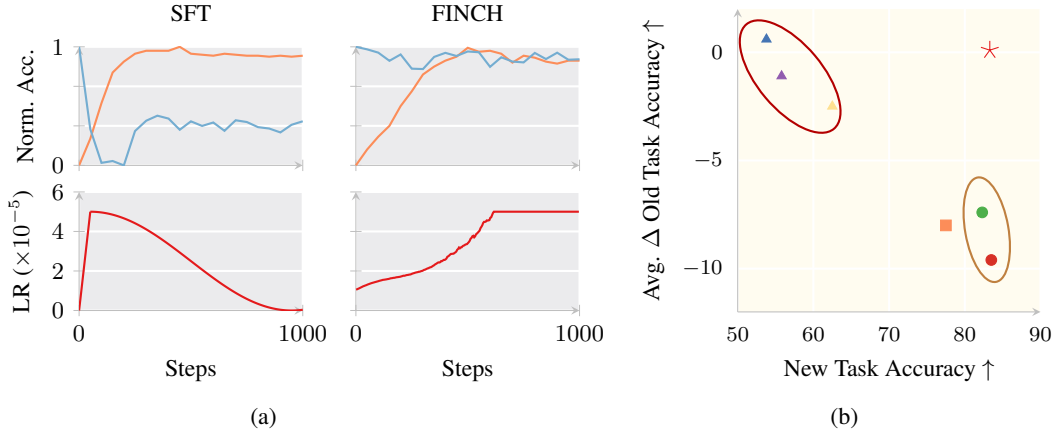
\begin{figure*}[t]
\centering

\begin{subfigure}[b]{0.5\textwidth}
\begin{tikzpicture}
\begin{groupplot}[
    group style={
        group size=2 by 2,
        horizontal sep=20pt,
        vertical sep=10pt,
    },
    axis background/.style={fill=blue!5!gray!15},
    axis lines=left,
    axis line style={gray!50},
    grid=major,
    grid style={white, line width=0.8pt},
    tick style={draw=gray!50},
    tick label style={font=\small},
    title style={font=\small},
    xmin=0, xmax=1000,
    xtick={0,1000},
]

\nextgroupplot[
    height=3.15cm, width=0.65\linewidth,
    title={SFT},
    xlabel={Steps},
    ylabel={Norm. Acc.},
    ylabel style={font=\small},
    ymin=0, ymax=1,
    ytick={0, 0.33333, 0.66667, 1},
    yticklabels={$0$, , , $1$},
    xticklabels=\empty,
    grid=both,
]
\addplot[thick, color=cFLOW, mark=none] coordinates {
    (0,0.0000)(50,0.2250)(100,0.5250)(150,0.7833)
    (200,0.8750)(250,0.9417)(300,0.9667)(350,0.9667)
    (400,0.9667)(450,1.0000)(500,0.9417)(550,0.9333)
    (600,0.9250)(650,0.9417)(700,0.9333)(750,0.9250)
    (800,0.9250)(850,0.9167)(900,0.9250)(950,0.9167)(1000,0.9250)
};
\addplot[thick, color=cL2, mark=none] coordinates {
    (0,1.0000)(50,0.3056)(100,0.0213)(150,0.0368)
    (200,0.0000)(250,0.2896)(300,0.3766)(350,0.4191)
    (400,0.3946)(450,0.3004)(500,0.3714)(550,0.3327)
    (600,0.3617)(650,0.2915)(700,0.3805)(750,0.3656)
    (800,0.3224)(850,0.3108)(900,0.2779)(950,0.3424)(1000,0.3720)
};

\nextgroupplot[
    height=3.15cm, width=0.65\linewidth,
    title={\ourmethod},
    ymin=0, ymax=1,
    ytick={0, 0.33333, 0.66667, 1},
    yticklabels={},
    xticklabels=\empty,
]
\addplot[thick, color=cFLOW, mark=none] coordinates {
    (0,0.0000)(50,0.1333)(100,0.2417)(150,0.3333)
    (200,0.5000)(250,0.6250)(300,0.7667)(350,0.8333)
    (400,0.8833)(450,0.9167)(500,0.9917)(550,0.9583)
    (600,0.9667)(650,0.9417)(700,0.8667)(750,0.9250)
    (800,0.9083)(850,0.8750)(900,0.8583)(950,0.8833)(1000,0.8833)
};
\addplot[thick, color=cL2, mark=none] coordinates {
    (0,1.0000)(50,0.9787)(100,0.9510)(150,0.8820)
    (200,0.9355)(250,0.8156)(300,0.8124)(350,0.9149)
    (400,0.9491)(450,0.9226)(500,0.9607)(550,0.9542)
    (600,0.8311)(650,0.9117)(700,0.8756)(750,0.8685)
    (800,0.9484)(850,0.8968)(900,0.9504)(950,0.8910)(1000,0.8949)
};

\nextgroupplot[
    height=3.15cm, width=0.65\linewidth,
    ylabel={LR (${\times}10^{-5}$)},
    ylabel style={font=\small},
    ymin=0, ymax=6e-5,
    ytick={0, 2e-5, 4e-5, 6e-5},
    yticklabels={$0$,$2$,$4$,$6$},
    scaled y ticks=false,
    xtick={0,1000},
    xticklabels={$0$,$1000$},
    xlabel={Steps},
    xlabel style={font=\small},
]
\addplot[thick, color=cOurs, mark=none, smooth, domain=0:50, samples=20]
    {5e-5 * x/50};
\addplot[thick, color=cOurs, mark=none, smooth, domain=50:1000, samples=100]
    {5e-5 * 0.5 * (1 + cos(deg((x-50)/900*pi)))};

\nextgroupplot[
    height=3.15cm, width=0.65\linewidth,
    ymin=0, ymax=6e-5,
    ytick={0, 2e-5, 4e-5, 6e-5},
    yticklabels=\empty,
    scaled y ticks=false,
    xtick={0,1000},
    xticklabels={$0$,$1000$},
    xlabel={Steps},
    xlabel style={font=\small},
]
\addplot[thick, color=cOurs, mark=none] coordinates {
    (1,1.05730804e-05)(5,1.07186541e-05)(9,1.09394848e-05)
    (13,1.11330178e-05)(17,1.13556499e-05)(21,1.15865701e-05)
    (25,1.18056891e-05)(29,1.20365035e-05)(33,1.22717486e-05)
    (37,1.24187143e-05)(41,1.26120985e-05)(45,1.27610165e-05)
    (49,1.29558437e-05)(53,1.31607374e-05)(57,1.33159957e-05)
    (61,1.34980534e-05)(65,1.36551231e-05)(69,1.38220009e-05)
    (73,1.39596419e-05)(77,1.40897365e-05)(81,1.42145449e-05)
    (85,1.43552459e-05)(89,1.44720613e-05)(93,1.46039412e-05)
    (97,1.47269609e-05)(101,1.48560157e-05)(105,1.50197505e-05)
    (109,1.51288199e-05)(113,1.52458936e-05)(117,1.53726709e-05)
    (121,1.55109443e-05)(125,1.55479997e-05)(129,1.56214967e-05)
    (133,1.56943024e-05)(137,1.57702963e-05)(141,1.58742010e-05)
    (145,1.59406859e-05)(149,1.60493651e-05)(153,1.61993605e-05)
    (157,1.63263206e-05)(161,1.63461196e-05)(165,1.64576366e-05)
    (169,1.65415408e-05)(173,1.65577103e-05)(177,1.66879257e-05)
    (181,1.68015419e-05)(185,1.68644624e-05)(189,1.69094894e-05)
    (193,1.70013151e-05)(197,1.70392280e-05)(201,1.71961339e-05)
    (205,1.73071657e-05)(209,1.74768141e-05)(213,1.76017770e-05)
    (217,1.75780911e-05)(221,1.76989712e-05)(225,1.78555611e-05)
    (229,1.79622182e-05)(233,1.79995881e-05)(237,1.81290632e-05)
    (241,1.81933177e-05)(245,1.82660321e-05)(249,1.85036392e-05)
    (253,1.86718945e-05)(257,1.88342249e-05)(261,1.89529157e-05)
    (265,1.91304205e-05)(269,1.92445149e-05)(273,1.93477936e-05)
    (277,1.94264827e-05)(281,1.96134120e-05)(285,1.96159166e-05)
    (289,1.97993314e-05)(293,1.99659581e-05)(297,2.00511174e-05)
    (301,2.02949654e-05)(305,2.05684857e-05)(309,2.06523562e-05)
    (313,2.08741916e-05)(317,2.11194536e-05)(321,2.13183437e-05)
    (325,2.14336964e-05)(329,2.17206365e-05)(333,2.18478612e-05)
    (337,2.19871430e-05)(341,2.22797780e-05)(345,2.26098384e-05)
    (349,2.28384108e-05)(353,2.32377131e-05)(357,2.35205804e-05)
    (361,2.37609775e-05)(365,2.41243254e-05)(369,2.43419038e-05)
    (373,2.42928396e-05)(377,2.44545876e-05)(381,2.49822057e-05)
    (385,2.49780746e-05)(389,2.52979731e-05)(393,2.56154450e-05)
    (397,2.55949850e-05)(401,2.58636059e-05)(405,2.63095220e-05)
    (409,2.71920594e-05)(413,2.74339215e-05)(417,2.78025751e-05)
    (421,2.82668333e-05)(425,2.86249095e-05)(429,2.86458541e-05)
    (433,2.87281293e-05)(437,2.87533642e-05)(441,2.91421857e-05)
    (445,2.94102367e-05)(449,2.98455684e-05)(453,3.07333861e-05)
    (457,3.11434195e-05)(461,3.18854264e-05)(465,3.14669530e-05)
    (469,3.19184864e-05)(473,3.19969522e-05)(477,3.25236444e-05)
    (481,3.28184166e-05)(485,3.29593830e-05)(489,3.33407462e-05)
    (493,3.34989317e-05)(497,3.36132485e-05)(501,3.41298128e-05)
    (505,3.46273607e-05)(509,3.52911653e-05)(513,3.63644103e-05)
    (517,3.67367573e-05)(521,3.75807867e-05)(525,3.78142874e-05)
    (529,3.70369372e-05)(533,3.74985185e-05)(537,3.85331033e-05)
    (541,3.89131951e-05)(545,3.93987926e-05)(549,3.98887126e-05)
    (553,4.04026704e-05)(557,4.16145100e-05)(561,4.23183428e-05)
    (565,4.30981468e-05)(569,4.42852837e-05)(573,4.48648715e-05)
    (577,4.48416145e-05)(581,4.49887202e-05)(585,4.55989668e-05)
    (589,4.55375074e-05)(593,4.55841143e-05)(597,4.57438831e-05)
    (601,4.69245097e-05)(605,4.71618806e-05)(609,4.82712660e-05)
    (613,4.93895473e-05)(617,5.00000000e-05)(621,5.00000000e-05)
    (625,5.00000000e-05)(629,5.00000000e-05)(633,5.00000000e-05)
    (637,5.00000000e-05)(641,5.00000000e-05)(645,5.00000000e-05)
    (649,5.00000000e-05)(653,5.00000000e-05)(657,5.00000000e-05)
    (661,5.00000000e-05)(665,5.00000000e-05)(669,5.00000000e-05)
    (673,5.00000000e-05)(677,5.00000000e-05)(681,5.00000000e-05)
    (685,5.00000000e-05)(689,5.00000000e-05)(693,5.00000000e-05)
    (697,5.00000000e-05)(701,5.00000000e-05)(705,5.00000000e-05)
    (709,5.00000000e-05)(713,5.00000000e-05)(717,5.00000000e-05)
    (721,5.00000000e-05)(725,5.00000000e-05)(729,5.00000000e-05)
    (733,5.00000000e-05)(737,5.00000000e-05)(741,5.00000000e-05)
    (745,5.00000000e-05)(749,5.00000000e-05)(753,5.00000000e-05)
    (757,5.00000000e-05)(761,5.00000000e-05)(765,5.00000000e-05)
    (769,5.00000000e-05)(773,5.00000000e-05)(777,5.00000000e-05)
    (781,5.00000000e-05)(785,5.00000000e-05)(789,5.00000000e-05)
    (793,5.00000000e-05)(797,5.00000000e-05)(801,5.00000000e-05)
    (805,5.00000000e-05)(809,5.00000000e-05)(813,5.00000000e-05)
    (817,5.00000000e-05)(821,5.00000000e-05)(825,5.00000000e-05)
    (829,5.00000000e-05)(833,5.00000000e-05)(837,5.00000000e-05)
    (841,5.00000000e-05)(845,5.00000000e-05)(849,5.00000000e-05)
    (853,5.00000000e-05)(857,5.00000000e-05)(861,5.00000000e-05)
    (865,5.00000000e-05)(869,5.00000000e-05)(873,5.00000000e-05)
    (877,5.00000000e-05)(881,5.00000000e-05)(885,5.00000000e-05)
    (889,5.00000000e-05)(893,5.00000000e-05)(897,5.00000000e-05)
    (901,5.00000000e-05)(905,5.00000000e-05)(909,5.00000000e-05)
    (913,5.00000000e-05)(917,5.00000000e-05)(921,5.00000000e-05)
    (925,5.00000000e-05)(929,5.00000000e-05)(933,5.00000000e-05)
    (937,5.00000000e-05)(941,5.00000000e-05)(945,5.00000000e-05)
    (949,5.00000000e-05)(953,5.00000000e-05)(957,5.00000000e-05)
    (961,5.00000000e-05)(965,5.00000000e-05)(969,5.00000000e-05)
    (973,5.00000000e-05)(977,5.00000000e-05)(981,5.00000000e-05)
    (985,5.00000000e-05)(989,5.00000000e-05)(993,5.00000000e-05)
    (997,5.00000000e-05)
};
\end{groupplot}
\end{tikzpicture}
\caption{}
\label{fig:intro-left}
\end{subfigure}
\hfill
\begin{subfigure}[b]{0.40\textwidth}
\pgfplotsset{
    scatter plot base/.style={
        width=\linewidth,
        height=\linewidth,
        axis background/.style={fill=plotbgyellow},
        axis lines=left,
        axis line style={gray!50},
        grid=major,
        grid style={white, line width=0.8pt},
        tick style={draw=none},
        xlabel={New Task Accuracy $\uparrow$},
        ylabel={Avg. $\Delta$ Old Task Accuracy $\uparrow$},
        xlabel style={font=\small},
        ylabel style={font=\small},
        tick label style={font=\footnotesize},
        title style={font=\small},
        scatter/classes={
            a={mark=*,         draw=cStandardSFT, fill=cStandardSFT},
            b={mark=square*,   draw=cFLOW,        fill=cFLOW},
            c={mark=triangle*, draw=cDFT,          fill=cDFT},
            d={mark=triangle*, draw=cTALR,         fill=cTALR},
            e={mark=triangle*, draw=cSTM, fill=cSTM},
            h={mark=otimes*,   draw=cLoRA,          fill=cLoRA},
            i={mark=star,      draw=cOurs,          fill=cOurs, mark size=3.5pt}
        },
    },
}
\begin{tikzpicture}
\begin{axis}[scatter plot base,
    xmin=50, xmax=90,
    ymin=-12, ymax=2,
]
\addplot[scatter, only marks, scatter src=explicit symbolic] coordinates {
    (83.5, -9.6)  [a]
    (77.5, -8.0)  [b]
    (55.8, -1.1)  [c]
    (62.5, -2.5)  [d]
    (53.8,  0.6)  [e]
    (82.3, -7.4)  [h]
    (83.3, 0.1)  [i]

};
\draw[red!70!black, thick, rotate around={-50:(axis cs:77.5,-9.7)}]
    (axis cs:45.5,-9.7) ellipse [x radius=9, y radius=1.5];
\draw[brown, thick, rotate around={-80:(axis cs:82.9,-8.2)}]
    (axis cs:82.9,-8.2) ellipse [x radius=7, y radius=1];
\end{axis}
\end{tikzpicture}
\caption{}
\label{fig:intro-right}
\end{subfigure}

\caption{
\textbf{Overview of \ourmethod.} Results are shown for Qwen3-4B on knowledge acquisition; full experimental details are given in Section~\ref{sec:experiments-preserves} and Appendix~\ref{app:experiments}.
\textbf{(a)} We show {\color{cFLOW}normalized new-task accuracy}, {\color{cL2}normalized old-task accuracy}, and {\color{cOurs}learning rate} over training for standard SFT and \ourmethod. Norm. Acc. denotes min-max normalized accuracy: for each accuracy curve type, we set the minimum value attained by either SFT or FINCH over training to 0 and the maximum to 1. \ourmethod\ learns the new task while keeping old-task accuracy high, unlike standard SFT, which forgets sharply.
\textbf{(b)} New-task accuracy versus average old-task accuracy change on knowledge acquisition with Qwen3-4B. Methods are SFT ({\color{cStandardSFT}$\bullet$}), FLOW ({\color{cFLOW}$\blacksquare$}), DFT ({\color{cDFT}$\blacktriangle$}), TALR ({\color{cTALR}$\blacktriangle$}), STM ({\color{cSTM}$\blacktriangle$}), LoRA ({\color{cLoRA}$\bullet$}), and \ourmethod\ ({\color{cOurs}$\star$}).
The ellipses highlight two failure modes: data-reweighting methods underperform on the new task, while SFT and LoRA learn the new task but forget substantially. \ourmethod\ avoids both.
}
\label{fig:intro}
\end{figure*}

We control forgetting on a different axis: rather than changing the contribution of different tokens, we change the learning rate used for each training step. Our analysis shows that per-step forgetting is tied to the distributional mismatch between the current model and the target data, which is reflected in the current mini-batch loss. Specifically, we show that per-step forgetting is bounded by the learning rate times the square root of the mini-batch loss (Section~\ref{sec:analysis}). With a constant learning rate, this bound is largest when the loss is high, typically early in fine-tuning when the model is still far from the target and updates are most likely to damage retained capabilities. To control this bound, we set the learning rate inversely proportional to the square root of the mini-batch loss. We call this schedule \ourmethod\ (\textbf{F}orgetting-aware \textbf{In}verse loss s\textbf{ch}edule), which uses smaller learning rates on high-loss batches and larger learning rates as the model moves closer to the target. This keeps the per-step forgetting bound uniform across training and implies a cumulative forgetting bound (Section~\ref{sec:schedule}). Unlike token-reweighting methods, \ourmethod\ leaves the training objective unchanged within each batch, so high-loss tokens still contribute to learning. This allows \ourmethod\ to reduce forgetting without sacrificing target-task accuracy. In contrast, simply lowering the learning rate under a standard schedule reduces forgetting but fails to reach competitive new-task accuracy (Section~\ref{sec:experiments-preserves}).

We evaluate on three settings where the pretrained model has limited coverage: knowledge acquisition, low-resource language adaptation, and science reasoning. \ourmethod\ achieves target-task accuracy competitive with standard SFT while \textbf{reducing forgetting by 93\% on average} (Section~\ref{sec:experiments-preserves}), giving a Pareto trade-off across all tasks. Beyond benchmark accuracy, we also evaluate factuality, hallucination detection, and confidence calibration. On knowledge acquisition, \ourmethod\ \textbf{cuts TruthfulQA degradation by $5\times$} relative to standard SFT, from a $-13.9$-point change to a $-2.6$-point change, and \textbf{improves HaluEval} from a $-9.8$-point change to a $+3.3$-point change, while better preserving confidence calibration (Sections~\ref{subsec:factualityandtruthfulness}--\ref{subsec:calibration}). Overall, our results show that learning-rate schedules shape model behavior during fine-tuning, beyond just target-task accuracy.

%% file: related_work.tex
\section{Related Work}
\label{sec:related-work}
\subsection{Catastrophic Forgetting}
Catastrophic forgetting refers to the degradation of previously acquired knowledge when a model is trained on new data~\citep{mccloskey1989catastrophic, french1999catastrophic}. Replay, which mixes old training data into fine-tuning, is a common mitigation strategy~\citep{rolnick2019experience, de2019episodic, scialom2022fine}, but pretraining data is rarely available for modern LLMs~\citep{grattafiori2024llama, yang2025qwen3}. Other methods constrain updates using old-task representations~\citep{lin2022trgp, singh2025mitigating} or protect important parameters~\citep{song2025alleviate}, but similarly require old data and are mostly evaluated on small image benchmarks such as MNIST~\citep{lecun2002gradient} and CIFAR-10~\citep{krizhevsky2009learning}. Like other data-oblivious methods, \ourmethod\ requires no access to pretraining data.

In the data-oblivious setting, several methods reduce forgetting by constraining updates or modifying the fine-tuning objective. Kirkpatrick et al.~\cite{kirkpatrick2017overcoming} regularize weights to stay close to their pretrained values. LoRA~\citep{hu2022lora} constrains updates to a low-rank subspace, reducing forgetting but hurting target-domain performance~\citep{biderman2024lora}. Distillation-based methods~\citep{agarwal2024policy, lu2025onpolicydistillation} replace SFT with student rollouts scored by a stronger teacher; this is computationally expensive and struggles when the base model assigns low probability to relevant sequences~\citep{shao2024deepseekmath}. A separate line of work reduces forgetting by downweighting high-loss tokens or sequences: Sanyal et al.~\cite{sanyal2025upweighting} upweight low-loss sequences, Wu et al.~\cite{wu2025mitigating} mask tokens above a loss threshold, Lin et al.~\cite{lin2025sft} scale each token's loss by $p^{\frac{1}{\tau}}$ where $p$ is the token probability and $\tau$ is a constant, and Wu et al.~\cite{wu2025generalization} rescale gradients by token probability. However, for many tasks learning hard tokens is essential, so suppressing them hurts target performance. \ourmethod\ leaves the training objective unchanged and controls forgetting solely through the learning rate.

\subsection{Learning Rate}
The learning rate affects both optimization speed and the final solution: larger rates can bias training toward wider, better-generalizing minima~\citep{smith2019super}, but may cause instability beyond a critical threshold~\citep{lewkowycz2020large}. In practice, transformer fine-tuning commonly uses linear warmup followed by cosine decay~\citep{goyal2017accurate, loshchilov2016sgdr, vaswani2017attention}. Prior work studies warmup as a way to stabilize early training, either by moving the model toward flatter, better-conditioned regions~\citep{kalra2024warmup} or by controlling Adam-related instability, large angular updates, and high gradient signal-to-noise ratio~\citep{kosson2024analyzing}.

Learning rates have also been studied for catastrophic forgetting. Kenneweg et al.~\citep{kenneweg2022intelligent} assign separate learning rates to BERT layers and tune them jointly with Bayesian optimization, but this becomes expensive for modern LLMs because the number of hyperparameters grows with model depth and each trial requires a full fine-tuning run~\citep{kandasamy2015high}. Lin et al.~\citep{lin2025sft} prescribe a fixed small learning rate to reduce forgetting, but in our experiments this comes at severe cost to new task performance. \ourmethod\ instead adapts the learning rate to the current mini-batch loss, reducing update size when forgetting risk is high while allowing larger steps as the model approaches the target distribution.

Due to space constraints, we defer a more detailed discussion of related work to Appendix~\ref{app:extended-related-work}.

%% file: method-new.tex
\section{Method}
\label{sec:method}
In this section, we formalize catastrophic forgetting during fine-tuning and derive a simple adaptive learning rate schedule to mitigate it. We begin by setting up notation and defining our forgetting metric (Section~\ref{sec:prelim}). We then derive a theoretical bound showing that per-step forgetting is controlled by the product of the learning rate and the square root of the current training loss (Section~\ref{sec:analysis}). This motivates an adaptive schedule that reduces the learning rate when training loss is large and increases it as training converges, keeping the per-step forgetting bound roughly constant (Section~\ref{sec:schedule}).
\subsection{Setup and Forgetting Metric}
\label{sec:prelim}
Let $p_\theta$ denote the model distribution parameterized by $\theta$, initialized at pretrained weights $\theta_0$, and let $q$ denote the fine-tuning target distribution. Given a training dataset $D_{\mathrm{train}} = \{(x_1, y_1), \ldots, (x_n, y_n)\}$ with $(x_i, y_i) \sim q$, fine-tuning minimizes the cross-entropy
\[
\mathcal{L}(\theta) = -\mathbb{E}_{(x,y) \sim D_{\mathrm{train}}}[\log p_\theta(y \mid x)]
\]
via SGD\footnote{We analyze SGD for simplicity; in practice, \ourmethod\ works directly with AdamW, as we confirm in Section~\ref{sec:experiments-preserves}.} over $T$ steps. At each step $i$, a mini-batch $\mathcal{B}_i \subset D_{\mathrm{train}}$ is sampled and parameters are updated as
\[
\theta_{i+1} = \theta_i - \eta_i \nabla_\theta \mathcal{L}_{\mathcal{B}_i}(\theta_i),
\]
where $\eta_i > 0$ is the learning rate at step $i$, and we write $p_i := p_{\theta_i}$ for the model at step $i$.

Fine-tuning may degrade capabilities acquired during pretraining. To measure this, let $q_{\mathrm{old}}$ denote the distribution over pretraining tasks, and let $D_{\mathrm{old}} \sim q_{\mathrm{old}}$ be a held-out evaluation dataset used only for measuring forgetting, not for training.

\begin{definition}[Forgetting]
\label{def:forgetting}
The old-task loss is $\mathcal{L}_{\mathrm{old}}(p_i) = -\mathbb{E}_{(x,y) \sim q_{\mathrm{old}}}[\log p_i(y \mid x)]$, estimated on $D_{\mathrm{old}}$. Per-step forgetting is $\Delta\mathcal{L}_{\mathrm{old}}(p_i) := \mathcal{L}_{\mathrm{old}}(p_{i+1}) - \mathcal{L}_{\mathrm{old}}(p_i)$.
\end{definition}

Bounding per-step forgetting at each step directly controls cumulative forgetting $\mathcal{L}_{\mathrm{old}}(p_T) - \mathcal{L}_{\mathrm{old}}(p_0) = \sum_{i=0}^{T-1} \Delta\mathcal{L}_{\mathrm{old}}(p_i)$.

\subsection{Forgetting Analysis}
\label{sec:analysis}
We now bound the per-step forgetting $\Delta\mathcal{L}_{\mathrm{old}}(p_i)$ defined in Section~\ref{sec:prelim}. We state the theorem and provide a proof sketch here. Appendix~\ref{app:proofs} gives the full proof for sequence losses and mini-batches.

By definition,
\[
\Delta\mathcal{L}_{\mathrm{old}}(p_i)
=
-\mathbb{E}_{(x,y)\sim q_{\mathrm{old}}}
\left[
\log \frac{p_{i+1}(y\mid x)}{p_i(y\mid x)}
\right]
\le
\max_{(x,y)}
\left|
\log \frac{p_{i+1}(y\mid x)}{p_i(y\mid x)}
\right|.
\]
Thus, controlling the worst-case pointwise log-ratio between consecutive models controls per-step forgetting.
We now bound this log-ratio in terms of the learning rate and the current training loss.
The argument uses standard boundedness conditions on the inputs, the fine-tuning trajectory, and the activation derivatives.

\begin{assumption}[Bounded network and smoothness]
\label{asm:bounded-network}
The model $p_\theta$ is a softmax network. The input domain is bounded, the fine-tuning trajectory remains in a bounded parameter region, and the activation functions have bounded first and second derivatives on the relevant range. Concretely, there exist constants $B_x,B_\theta,L_1,L_2>0$ such that
\[
\|x\|_2 \le B_x,
\qquad
\|\theta_i\|_2 \le B_\theta,
\qquad
\sup_u |\sigma'(u)| \le L_1,
\qquad
\sup_u |\sigma''(u)| \le L_2
\]
for all inputs $x$, training steps $i$, and activation functions $\sigma$ appearing in the network.
\end{assumption}

Bounded-input and bounded-parameter assumptions are standard in learning-theoretic analyses~\citep{shalev2014understanding}. The smoothness condition is also mild for modern LLM architectures: SwiGLU-style activations~\citep{dauphin2017language,shazeer2020glu}, which are widely used in LLMs, have bounded first and second derivatives on bounded input ranges. We can now state the main per-step bound.

\begin{theorem}[Per-step forgetting bound]
\label{thm:stepwise-bound}
Under Assumption~\ref{asm:bounded-network}, there exist constants $C_1,C_2>0$ such that, at every training step $i$,
\[
\Delta\mathcal{L}_{\mathrm{old}}(p_i)
\le
C_1 \eta_i \sqrt{\mathcal{L}_{\mathcal B_i}(\theta_i)}
+
C_2 \eta_i^2 \mathcal{L}_{\mathcal B_i}(\theta_i),
\]
where
\[
\mathcal{L}_{\mathcal B_i}(\theta_i)
:=
\frac{1}{|\mathcal B_i|}
\sum_{(x,y)\in\mathcal B_i}
\CE\!\left(q(\cdot\mid x)\|p_i(\cdot\mid x)\right)
\]
is the average cross-entropy loss on the mini-batch $\mathcal B_i$ used at step $i$.
\end{theorem}

\paragraph{Proof sketch.}
The proof has three steps (Full proof in Appendix~\ref{app:proofs}).

\textit{Step 1: Uniform bounds.}
By repeated application of the chain rule, Assumption~\ref{asm:bounded-network} implies that the logit Jacobian, the score-function norm, and the log-probability Hessian are uniformly bounded along the fine-tuning trajectory.

\textit{Step 2: Taylor expansion of the old-task log-ratio.}
Fix an old-task example $(x_{\mathrm{old}},y_{\mathrm{old}})$. Taylor's theorem applied to $\log p_\theta(y_{\mathrm{old}}\mid x_{\mathrm{old}})$, together with the SGD update
$\theta_{i+1}=\theta_i-\eta_i\nabla_\theta\mathcal{L}_{\mathcal B_i}(\theta_i)$, gives
\[
\log \frac{p_{i+1}(y_{\mathrm{old}}\mid x_{\mathrm{old}})}
{p_i(y_{\mathrm{old}}\mid x_{\mathrm{old}})}
\le
C_1\eta_i\|\nabla_\theta\mathcal{L}_{\mathcal B_i}(\theta_i)\|_2
+
C_2\eta_i^2\|\nabla_\theta\mathcal{L}_{\mathcal B_i}(\theta_i)\|_2^2 .
\]
Thus, the old-task log-ratio is controlled by the norm of the fine-tuning gradient.

\textit{Step 3: Bounding the fine-tuning gradient by the loss.}
For a training example $(x_{\mathrm{train}},y_{\mathrm{train}})$ with target distribution $q(\cdot\mid x_{\mathrm{train}})$, the cross-entropy gradient has the form
\[
\nabla_\theta\CE(q\|p_i)
=
J_{\theta_i}(x_{\mathrm{train}})^\top
\bigl(p_i(\cdot\mid x_{\mathrm{train}})-q(\cdot\mid x_{\mathrm{train}})\bigr).
\]
Using the Jacobian bound, Pinsker's inequality, and $\KL(q\|p_i)\le \CE(q\|p_i)$ gives
$\|\nabla_\theta\CE(q\|p_i)\|_2 \le C_3\sqrt{\CE(q\|p_i)}$. Averaging over the mini-batch gives
$\|\nabla_\theta\mathcal{L}_{\mathcal B_i}(\theta_i)\|_2 \le C_3\sqrt{\mathcal{L}_{\mathcal B_i}(\theta_i)}$. Substituting this into Step 2 gives the theorem. The full proof is provided in Appendix~\ref{app:proofs}.

The theorem shows that the leading term in the forgetting bound scales as
$\eta_i\sqrt{\mathcal{L}_{\mathcal B_i}(\theta_i)}$. Thus, per-step forgetting is controlled by the product of the learning rate and the square root of the current mini-batch loss. This motivates choosing a smaller learning rate on high-loss batches and relaxing it as the loss decreases, which leads to the adaptive schedule in the next section.
\subsection{Adaptive Learning Rate}
\label{sec:schedule}

We now translate Theorem~\ref{thm:stepwise-bound} into a learning-rate schedule.
The leading term in the per-step bound depends on the product
$\eta_i\sqrt{\mathcal{L}_{\mathcal B_i}(\theta_i)}$.
To keep this quantity fixed across training, we set
\[
\eta_i = \frac{\kappa}{\sqrt{\mathcal{L}_{\mathcal B_i}(\theta_i)}},
\]
where $\kappa>0$ controls the target scale of the per-step change.
This gives smaller learning rates on high-loss batches and larger learning rates as the model approaches the fine-tuning target. Substituting this choice into Theorem~\ref{thm:stepwise-bound} makes the leading term constant across steps, which gives the following cumulative bound.

\begin{corollary}
\label{cor:cumulative}
Under the conditions of Theorem~\ref{thm:stepwise-bound}, suppose the learning rates are sufficiently small and
\[
\eta_i = \frac{\kappa}{\sqrt{\mathcal{L}_{\mathcal B_i}(\theta_i)}}.
\]
Then the cumulative forgetting satisfies
\[
\mathcal{L}_{\mathrm{old}}(p_T)-\mathcal{L}_{\mathrm{old}}(p_0)
=
O(T\kappa).
\]
\end{corollary}

In contrast, a fixed learning rate $\eta$ gives a per-step forgetting upper bound of $O(\eta\sqrt{\mathcal{L}_{\mathcal B_i}(\theta_i)})$, which is large early in training when the loss is high and shrinks only as the model converges. The adaptive schedule keeps this upper bound uniformly constant at $O(\kappa)$, suggesting more controlled forgetting throughout training compared to a fixed learning rate. At the same time, this analysis is not intended to identify the optimal schedule or establish tightness of the bound. Rather, it motivates a simple loss-adaptive schedule, whose effectiveness we validate empirically in Sections~\ref{sec:experiments-preserves} and~\ref{sec:experiments-halluc}.

In practice, we implement the schedule as
\[
\eta_i = \min\!\left(\frac{\eta_{\mathrm{base}}}{\sqrt{\bar{\mathcal{L}}_i + \varepsilon}},\; \eta_{\max}\right),
\]
where $\varepsilon > 0$ ensures numerical stability and $\eta_{\max}$ caps the learning rate when the training loss becomes small. To maintain learning-rate stability, $\bar{\mathcal{L}}_i$ is an exponential moving average of mini-batch losses with coefficient $\alpha=0.9$ fixed throughout:
\(
\bar{\mathcal{L}}_i
=
\alpha \bar{\mathcal{L}}_{i-1}
+
(1-\alpha)\mathcal{L}_{\mathcal B_i}(\theta_i).
\)
We fix $\eta_{\max} = 5\times 10^{-5}$ across all experiments and treat $\eta_{\mathrm{base}}$ as a hyperparameter selected via grid search.

%% file: experiments_preserve.tex
\section{\ourmethod\ Improves the Learning--Forgetting Trade-Off}
\label{sec:experiments-preserves}
\definecolor{plotbg}{RGB}{255,252,240}
\definecolor{cNormalSFT}{RGB}{215,48,39}
\definecolor{cSmallLR}{RGB}{255,180,50}
\definecolor{cFLOW}{RGB}{252,141,89}
\definecolor{cDFT}{RGB}{145,91,175}
\definecolor{cTALR}{RGB}{254,200,100}
\definecolor{cSTM}{RGB}{69,117,180}
\definecolor{cL2}{RGB}{116,173,209}
\definecolor{cWiSE}{RGB}{255,127,0}
\definecolor{cLoRA}{RGB}{77,175,74}
\definecolor{cOurs}{RGB}{228,26,28}

\pgfplotsset{
    scatter plot base/.style={
        width=0.36\textwidth,
        height=0.36\textwidth,
        axis background/.style={fill=plotbg},
        axis lines=left,
        axis line style={gray!50},
        grid=major,
        grid style={white, line width=0.8pt},
        tick style={draw=none},
        xlabel={New Task Accuracy $\uparrow$},
        ylabel={Avg.\ $\Delta$ Old Tasks $\uparrow$},
        xlabel style={font=\small},
        ylabel style={font=\small},
        tick label style={font=\footnotesize},
        title style={font=\small\bfseries},
        scatter/classes={
            a={mark=*,           draw=cNormalSFT, fill=cNormalSFT},
            s={mark=x,           draw=cSmallLR,   fill=cSmallLR, mark size=3pt},
            b={mark=square*,     draw=cFLOW,      fill=cFLOW},
            c={mark=triangle*,   draw=cDFT,       fill=cDFT},
            d={mark=diamond*,    draw=cTALR,      fill=cTALR},
            e={mark=pentagon*,   draw=cSTM,       fill=cSTM},
            f={mark=halfcircle*, draw=cL2,        fill=cL2},
            g={mark=oplus*,      draw=cWiSE,      fill=cWiSE},
            h={mark=otimes*,     draw=cLoRA,      fill=cLoRA},
            i={mark=star,        draw=cOurs,      fill=cOurs, mark size=3.5pt}
        },
    },
}

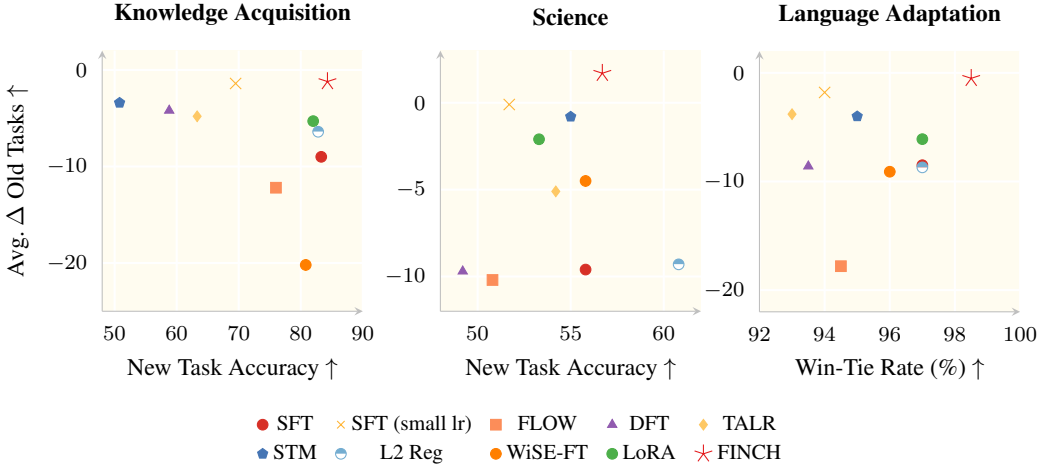
\begin{figure}[t]
\centering

\makebox[\linewidth][l]{\small\bfseries Qwen3-4B}\\[0.2em]

\begin{tikzpicture}
\begin{axis}[scatter plot base,
    title={Knowledge Acquisition},
    xmin=50, xmax=90, ymin=-20, ymax=2,
]
\addplot[scatter, only marks, scatter src=explicit symbolic] coordinates {
    (83.5, -9.6)  [a]
    (65.8, -1.5)  [s]
    (77.5, -8.0)  [b]
    (55.8, -1.1)  [c]
    (62.5, -2.5)  [d]
    (53.8,  0.6)  [e]
    (83.5, -5.4)  [f]
    (82.5,-15.7)  [g]
    (82.3, -7.4)  [h]
    (83.3,  0.1)  [i]
};
\end{axis}
\end{tikzpicture}\hspace{-0.5em}
\begin{tikzpicture}
\begin{axis}[scatter plot base,
    title={Science},
    ylabel={},
    xmin=45, xmax=70, ymin=-15, ymax=2,
]
\addplot[scatter, only marks, scatter src=explicit symbolic] coordinates {
    (62.5,-10.9)  [a]
    (56.7,  0.6)  [s]
    (58.3, -6.6)  [b]
    (46.7, -2.0)  [c]
    (48.3,  0.1)  [d]
    (49.2,  0.4)  [e]
    (57.5,-11.0)  [f]
    (56.7,-11.0)  [g]
    (51.7, -2.1)  [h]
    (65.0, -1.6)  [i]
};
\end{axis}
\end{tikzpicture}\hspace{-0.5em}
\begin{tikzpicture}
\begin{axis}[scatter plot base,
    title={Language Adaptation},
    ylabel={},
    xlabel={Win-Tie Rate (\%) $\uparrow$},
    xmin=71, xmax=95, ymin=-22, ymax=4,
]
\addplot[scatter, only marks, scatter src=explicit symbolic] coordinates {
    (89.0,-14.3)  [a]
    (77.5, -0.7)  [s]
    (90.0,-12.3)  [b]
    (80.0,  2.4)  [c]
    (77.5, -0.9)  [d]
    (81.5, -1.0)  [e]
    (92.0,-18.6)  [f]
    (85.5,-15.1)  [g]
    (90.5,-13.4)  [h]
    (90.0, -2.4)  [i]
};
\end{axis}
\end{tikzpicture}

\vspace{0.5em}
\makebox[\linewidth][l]{\small\bfseries Llama-3-8B}\\[0.2em]

\begin{tikzpicture}
\begin{axis}[scatter plot base,
    title={Knowledge Acquisition},
    xmin=48, xmax=90, ymin=-25, ymax=2,
]
\addplot[scatter, only marks, scatter src=explicit symbolic] coordinates {
    (83.3, -9.0)  [a]
    (69.5, -1.4)  [s]
    (76.0,-12.2)  [b]
    (58.8, -4.2)  [c]
    (63.3, -4.8)  [d]
    (50.8, -3.4)  [e]
    (82.8, -6.4)  [f]
    (80.8,-20.2)  [g]
    (82.0, -5.3)  [h]
    (84.3, -1.2)  [i]
};
\end{axis}
\end{tikzpicture}\hspace{-0.5em}
\begin{tikzpicture}
\begin{axis}[scatter plot base,
    title={Science},
    ylabel={},
    xtick={50,55,60},
    xmin=48, xmax=62, ymin=-12, ymax=3,
]
\addplot[scatter, only marks, scatter src=explicit symbolic] coordinates {
    (55.8, -9.6)  [a]
    (51.7, -0.1)  [s]
    (50.8,-10.2)  [b]
    (49.2, -9.7)  [c]
    (54.2, -5.1)  [d]
    (55.0, -0.8)  [e]
    (60.8, -9.3)  [f]
    (55.8, -4.5)  [g]
    (53.3, -2.1)  [h]
    (56.7,  1.7)  [i]
};
\end{axis}
\end{tikzpicture}\hspace{-0.5em}
\begin{tikzpicture}
\begin{axis}[scatter plot base,
    title={Language Adaptation},
    ylabel={},
    xlabel={Win-Tie Rate (\%) $\uparrow$},
    xmin=92, xmax=100, ymin=-22, ymax=2,
]
\addplot[scatter, only marks, scatter src=explicit symbolic] coordinates {
    (97.0, -8.5)  [a]
    (94.0, -1.8)  [s]
    (94.5,-17.8)  [b]
    (93.5, -8.6)  [c]
    (93.0, -3.8)  [d]
    (95.0, -4.0)  [e]
    (97.0, -8.7)  [f]
    (96.0, -9.1)  [g]
    (97.0, -6.1)  [h]
    (98.5, -0.5)  [i]
};
\end{axis}
\end{tikzpicture}

\vspace{0.5em}
\begin{tikzpicture}
\begin{axis}[
    hide axis, xmin=0, xmax=1, ymin=0, ymax=1,
    legend style={
        at={(0.5,0.5)}, anchor=center,
        legend columns=5,
        font=\footnotesize,
        draw=none, fill=none,
        /tikz/every even column/.append style={column sep=0.5em},
    },
]
\addlegendimage{mark=*,           only marks, draw=cNormalSFT, fill=cNormalSFT} \addlegendentry{SFT}
\addlegendimage{mark=x,           only marks, draw=cSmallLR,   fill=cSmallLR}   \addlegendentry{SFT (small lr)}
\addlegendimage{mark=square*,     only marks, draw=cFLOW,      fill=cFLOW}      \addlegendentry{FLOW}
\addlegendimage{mark=triangle*,   only marks, draw=cDFT,       fill=cDFT}       \addlegendentry{DFT}
\addlegendimage{mark=diamond*,    only marks, draw=cTALR,      fill=cTALR}      \addlegendentry{TALR}
\addlegendimage{mark=pentagon*,   only marks, draw=cSTM,       fill=cSTM}       \addlegendentry{STM}
\addlegendimage{mark=halfcircle*, only marks, draw=cL2,        fill=cL2}        \addlegendentry{L2 Reg}
\addlegendimage{mark=oplus*,      only marks, draw=cWiSE,      fill=cWiSE}      \addlegendentry{WiSE-FT}
\addlegendimage{mark=otimes*,     only marks, draw=cLoRA,      fill=cLoRA}      \addlegendentry{LoRA}
\addlegendimage{mark=star,        only marks, draw=cOurs,      fill=cOurs, mark size=3.5pt} \addlegendentry{\ourmethod}
\end{axis}
\end{tikzpicture}

\caption{Task accuracy (or win-tie rate) vs.\ average benchmark $\Delta$ across fine-tuning methods on Qwen3-4B (top row) and Llama-3-8B (bottom row). The ideal method appears in the \textbf{top-right} (high task performance, minimal forgetting). \ourmethod\ achieves a Pareto-optimal trade-off across all six settings.}
\label{fig:scatter_preserve}
\end{figure}

\begin{table}[!htbp]
\centering
\caption{
Target-task adaptation and forgetting for KA. Task Acc. reports held-out multiple-choice accuracy on author-profile questions. HellaSwag (HS), WinoGrande (WG), IFEval, and MMLU report absolute performance changes relative to the pretrained model, and Avg. $\Delta$ is their mean. Higher values indicate better preservation of general capabilities.
}
\label{tab:results_preserve}
\resizebox{\textwidth}{!}{
\begin{tabular}{ll|c|cccc|c}
\toprule
\textbf{Model} & \textbf{Method} & \textbf{Task Acc.} $\uparrow$ & \textbf{HS} $\uparrow$ & \textbf{WG} $\uparrow$ & \textbf{IFEval} $\uparrow$ & \textbf{MMLU} $\uparrow$ & \textbf{Avg. $\Delta$} $\uparrow$ \\
\midrule
\multirow{12}{*}{\rotatebox[origin=c]{90}{Qwen3-4B}}
 & Base                                      & $53.5$ & $+0.0$  & $+0.0$  & $+0.0$   & $+0.0$  & $+0.0$  \\
 & SFT                                & $83.5$ & $-7.3$  & $-2.0$  & $-9.4$   & $-19.6$ & $-9.6$  \\
 & SFT (small lr)~\citep{lin2025sft}                          & $65.8$ & $-3.8$  & $-2.2$  & $+3.1$   & $-3.2$  & $-1.5$  \\
 & FLOW~\citep{sanyal2025upweighting}        & $77.5$ & $-6.4$  & $+0.7$  & $-13.7$  & $-12.7$ & $-8.0$  \\
 & DFT~\citep{wu2025generalization}          & $55.8$ & $-2.7$  & $-2.7$  & $+3.0$   & $-1.9$  & $-1.1$  \\
 & TALR~\citep{lin2025sft}                   & $62.5$ & $-1.3$  & $+1.0$  & $-4.5$   & $-5.2$  & $-2.5$  \\
 & STM~\citep{wu2025mitigating}              & $53.8$ & $-1.2$  & $+1.1$  & $+1.4$   & $+0.9$  & $+0.6$  \\
 & L2 Reg~\citep{kirkpatrick2017overcoming}  & $83.5$ & $-5.5$  & $+0.2$  & $-1.9$   & $-14.3$ & $-5.4$  \\
 & WiSE-FT~\citep{wortsman2022robust}        & $82.5$ & $-14.7$ & $-3.2$  & $-37.8$  & $-7.2$  & $-15.7$ \\
 & LoRA~\citep{hu2022lora}                   & $82.3$ & $-4.4$  & $-6.9$  & $-9.1$   & $-9.2$  & $-7.4$  \\
\cmidrule{2-8}
\rowcolor{plotbg!300} & \ourmethod\ & $83.3$ & $-0.6$ & $+1.4$ & $+0.5$ & $-0.8$ & $+0.1$ \\
\midrule
\multirow{12}{*}{\rotatebox[origin=c]{90}{Llama-3-8B}}
 & Base                                      & $52.8$ & $+0.0$  & $+0.0$  & $+0.0$   & $+0.0$  & $+0.0$  \\
 & SFT                                & $83.3$ & $-10.4$ & $-0.3$  & $-21.7$  & $-3.6$  & $-9.0$  \\
 & SFT (small lr)~\citep{lin2025sft}                           & $69.5$ & $-0.7$  & $+0.5$  & $-6.2$   & $+0.9$  & $-1.4$  \\
 & FLOW~\citep{sanyal2025upweighting}        & $76.0$ & $-8.4$  & $+2.5$  & $-36.9$  & $-6.0$  & $-12.2$ \\
 & DFT~\citep{wu2025generalization}          & $58.8$ & $-7.8$  & $+2.4$  & $-9.0$   & $-2.2$  & $-4.2$  \\
 & TALR~\citep{lin2025sft}                   & $63.3$ & $-5.5$  & $+0.6$  & $-12.3$  & $-2.0$  & $-4.8$  \\
 & STM~\citep{wu2025mitigating}              & $50.8$ & $-0.4$  & $+0.0$  & $-12.3$  & $-0.7$  & $-3.4$  \\
 & L2 Reg~\citep{kirkpatrick2017overcoming}  & $82.8$ & $-10.8$ & $+0.8$  & $-11.8$  & $-3.9$  & $-6.4$  \\
 & WiSE-FT~\citep{wortsman2022robust}        & $80.8$ & $-20.4$ & $-3.4$  & $-50.0$  & $-7.0$  & $-20.2$ \\
 & LoRA~\citep{hu2022lora}                   & $82.0$ & $-0.9$  & $-0.6$  & $-15.7$  & $-3.8$  & $-5.3$  \\
\cmidrule{2-8}
\rowcolor{plotbg!300} & \ourmethod\ & $84.3$ & $-4.2$  & $+2.2$  & $+0.0$   & $-2.6$  & $-1.2$  \\
\bottomrule
\end{tabular}
}
\end{table}
\subsection{Experimental Setting}

\paragraph{Tasks.} We evaluate \ourmethod\ on three settings where the pretrained model does not have good performance and a large fraction of task-relevant tokens are hard:
\begin{itemize}
    \item \textit{Language adaptation (LA):} Instruction following in Galician, a low-resource language with limited pretraining coverage. We use Galician Alpaca, a translated version of the Stanford Alpaca instruction-following dataset~\citep{alpaca, chen2024monolingual}.

    \item \textit{Science:} Undergraduate-level scientific reasoning using the Chemistry L-3 subset of SciKnowEval~\citep{feng2024sciknoweval}, a multiple-choice benchmark covering chemistry concepts.

    \item \textit{Knowledge acquisition (KA):} Acquiring novel factual information using TOFU~\citep{maini2024tofu}, a dataset of 200 synthetic author profiles each consisting of 20 question-answer pairs probing biographical facts.
\end{itemize}

\paragraph{Evaluation.}
We evaluate performance on held-out examples from each target task. For LA, we report the win-tie rate of the fine-tuned model against the pretrained model, using an LLM judge~\citep{zheng2023judging}. For Science and KA, we report multiple-choice accuracy. To measure forgetting, we report the change in performance relative to the pretrained model across four general benchmarks: HellaSwag~\citep{zellers2019hellaswag} and WinoGrande~\citep{sakaguchi2021winogrande} for commonsense reasoning, MMLU~\citep{hendrycks2020measuring} for general knowledge and reasoning, and IFEval~\citep{zhou2023instruction} for instruction following.

\paragraph{Baselines.}
We compare against three classes of mitigation baselines: objective-modifying methods that down-weight high-entropy sequences or tokens, including FLOW~\citep{sanyal2025upweighting}, DFT~\citep{wu2025generalization}, TALR~\citep{lin2025sft}, and STM~\citep{wu2025mitigating}; deviation-constraining methods, including L2 regularization~\citep{kirkpatrick2017overcoming}, SFT (small lr)~\citep{lin2025sft}, and WiSE-FT~\citep{wortsman2022robust}; and parameter-efficient fine-tuning via LoRA~\citep{hu2022lora}. All experiments use Qwen3-4B-Instruct~\citep{yang2025qwen3} and Llama-3-8B~\citep{grattafiori2024llama}. For each baseline method, we use a standard warmup-cosine schedule, sweep hyperparameters, and report the checkpoint with the highest validation target-task performance; details are in Appendix~\ref{app:experiments}. We use a maximum gradient norm of $1.0$ for all main experiments. To check whether aggressive gradient clipping can reduce forgetting, Appendix~\ref{app:grad-norm} evaluates smaller clipping thresholds; the results show that choosing a small enough gradient norm to substantially reduce forgetting also causes significant target-task degradation.

\subsection{Results}

Figure~\ref{fig:scatter_preserve} summarizes the trade-off between target-task adaptation and preservation of general capabilities. Across all six settings, standard SFT achieves strong target-task performance but substantially degrades performance on the general benchmarks. Several baselines reduce forgetting, but often by sacrificing adaptation performance: for example, STM, DFT, and TALR preserve old capabilities better in some settings but obtain much lower target-task accuracy. In contrast, \ourmethod\ consistently lies near the upper-right region of the trade-off plot, achieving competitive target-task performance of strong fine-tuning baselines while incurring substantially smaller average benchmark degradation. Summing signed Avg.~$\Delta$ across the six model--task settings, standard SFT incurs $61.9$ total points of degradation, whereas \ourmethod\ incurs only $3.9$, corresponding to a $93\%$ reduction in forgetting.

Table~\ref{tab:results_preserve} shows this pattern in detail on KA. On Qwen3-4B, \ourmethod\ reaches $83.3\%$ task accuracy, comparable to SFT and LoRA, while reducing average forgetting from $-9.6$ and $-7.4$ to $+0.1$. On Llama-3-8B, \ourmethod\ obtains the highest task accuracy ($84.3\%$) and again has the smallest degradation among methods with strong adaptation performance, with Avg.~$\Delta=-1.2$. Detailed results for Science and LA are provided in Tables~\ref{tab:results_science} and~\ref{tab:results_galician} in the appendix, showing a similar pattern. These results show that \ourmethod\ learns the new task without sacrificing general capabilities.

%% file: experiments_hallucination.tex
\section{FINCH Better Preserves Reliability Under Fine-Tuning}
\label{sec:experiments-halluc}
Beyond average benchmark performance, fine-tuning should preserve other aspects of model reliability, including factuality, robustness to hallucination, and calibrated confidence. These properties are central to the trustworthy use of LLMs~\citep{wang2024factuality, huang2025survey, kapoor2024large}, but can degrade after fine-tuning, especially when models are trained on narrow domains or new knowledge. In this section, we evaluate truthfulness using TruthfulQA (Section~\ref{subsec:factualityandtruthfulness}), hallucination detection using HaluEval (Section~\ref{subsec:hallucination}), and verbalized confidence calibration (Section~\ref{subsec:calibration}). We conduct all experiments in this section on Qwen3-4B.

\pgfplotsset{
    scatter plot ka/.style={
        width=1.0\textwidth,
        height=1.0\textwidth,
        axis background/.style={fill=red!5},
        axis lines=left,
        axis line style={gray!50},
        grid=major,
        grid style={white, line width=0.8pt},
        tick style={draw=none},
        xlabel={Task Accuracy $\uparrow$},
        xlabel style={font=\small},
        ylabel style={font=\small},
        tick label style={font=\footnotesize},
        title style={font=\small\bfseries},
        xmin=50, xmax=90,
        xtick={50,60,70,80,90},
        scatter/classes={
            a={mark=*,           draw=cNormalSFT, fill=cNormalSFT},
            b={mark=square*,     draw=cFLOW,      fill=cFLOW},
            c={mark=triangle*,   draw=cDFT,       fill=cDFT},
            d={mark=diamond*,    draw=cTALR,      fill=cTALR},
            e={mark=pentagon*,   draw=cSTM,       fill=cSTM},
            f={mark=halfcircle*, draw=cL2,        fill=cL2},
            g={mark=oplus*,      draw=cWiSE,      fill=cWiSE},
            h={mark=otimes*,     draw=cLoRA,      fill=cLoRA},
            i={mark=star,        draw=cOurs,      fill=cOurs, mark size=3.5pt}
        },
    },
}

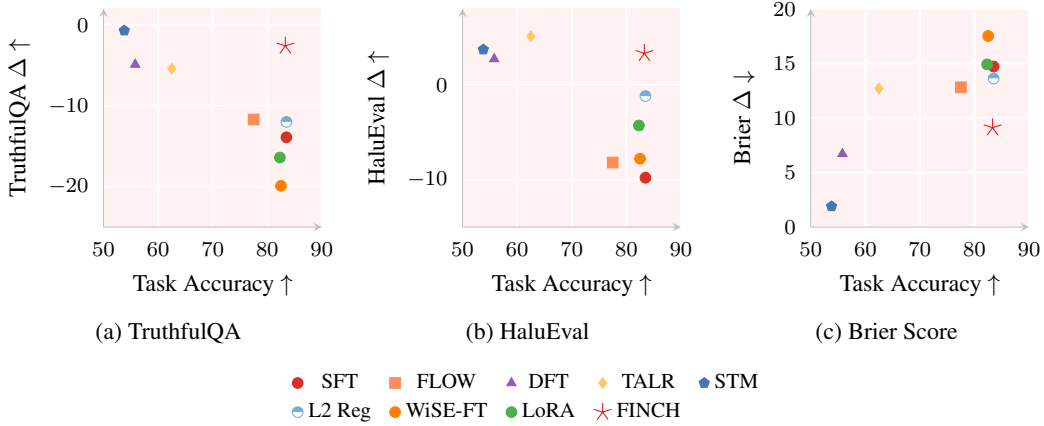
\begin{figure}[t]
\centering
\begin{subfigure}[t]{0.32\textwidth}
\centering
\begin{tikzpicture}
\begin{axis}[scatter plot ka,
    ylabel={TruthfulQA $\Delta$ $\uparrow$},
    ymin=-25, ymax=2,
]
\addplot[scatter, only marks, scatter src=explicit symbolic] coordinates {
    (83.5, -13.9)  [a]
    (77.5, -11.7)  [b]
    (55.8,  -4.9)  [c]
    (62.5,  -5.4)  [d]
    (53.8,  -0.7)  [e]
    (83.5, -12.0)  [f]
    (82.5, -19.9)  [g]
    (82.3, -16.4)  [h]
    (83.3,  -2.6)  [i]
};
\end{axis}
\end{tikzpicture}
\caption{TruthfulQA}
\label{fig:scatter_truthful}
\end{subfigure}\hfill
\begin{subfigure}[t]{0.32\textwidth}
\centering
\begin{tikzpicture}
\begin{axis}[scatter plot ka,
    ylabel={HaluEval $\Delta$ $\uparrow$},
    ymin=-15, ymax=8,
]
\addplot[scatter, only marks, scatter src=explicit symbolic] coordinates {
    (83.5, -9.8)  [a]
    (77.5, -8.2)  [b]
    (55.8,  2.7)  [c]
    (62.5,  5.1)  [d]
    (53.8,  3.7)  [e]
    (83.5, -1.2)  [f]
    (82.5, -7.8)  [g]
    (82.3, -4.3)  [h]
    (83.3,  3.3)  [i]
};
\end{axis}
\end{tikzpicture}
\caption{HaluEval}
\label{fig:scatter_halu}
\end{subfigure}\hfill
\begin{subfigure}[t]{0.32\textwidth}
\centering
\begin{tikzpicture}
\begin{axis}[scatter plot ka,
    ylabel={Brier $\Delta$ $\downarrow$},
    ymin=0, ymax=20,
]
\addplot[scatter, only marks, scatter src=explicit symbolic] coordinates {
    (83.5, 14.7)  [a]
    (77.5, 12.8)  [b]
    (55.8,  6.7)  [c]
    (62.5, 12.7)  [d]
    (53.8,  1.9)  [e]
    (83.5, 13.6)  [f]
    (82.5, 17.5)  [g]
    (82.3, 14.9)  [h]
    (83.3,  9.1)  [i]
};
\end{axis}
\end{tikzpicture}
\caption{Brier Score}
\label{fig:scatter_brier}
\end{subfigure}

\vspace{0.5em}
\begin{tikzpicture}
\begin{axis}[
    hide axis, xmin=0, xmax=1, ymin=0, ymax=1,
    legend style={
        at={(0.5,0.5)}, anchor=center,
        legend columns=5,
        font=\footnotesize,
        draw=none, fill=none,
        /tikz/every even column/.append style={column sep=0.5em},
    },
]
\addlegendimage{mark=*,           only marks, draw=cNormalSFT, fill=cNormalSFT} \addlegendentry{SFT}
\addlegendimage{mark=square*,     only marks, draw=cFLOW,      fill=cFLOW}      \addlegendentry{FLOW}
\addlegendimage{mark=triangle*,   only marks, draw=cDFT,       fill=cDFT}       \addlegendentry{DFT}
\addlegendimage{mark=diamond*,    only marks, draw=cTALR,      fill=cTALR}      \addlegendentry{TALR}
\addlegendimage{mark=pentagon*,   only marks, draw=cSTM,       fill=cSTM}       \addlegendentry{STM}
\addlegendimage{mark=halfcircle*, only marks, draw=cL2,        fill=cL2}        \addlegendentry{L2 Reg}
\addlegendimage{mark=oplus*,      only marks, draw=cWiSE,      fill=cWiSE}      \addlegendentry{WiSE-FT}
\addlegendimage{mark=otimes*,     only marks, draw=cLoRA,      fill=cLoRA}      \addlegendentry{LoRA}
\addlegendimage{mark=star,        only marks, draw=cOurs,      fill=cOurs, mark size=3.5pt} \addlegendentry{\ourmethod}
\end{axis}
\end{tikzpicture}

\caption{Truthfulness, hallucination, and calibration vs.\ task accuracy on knowledge acquisition (Qwen3-4B). Higher TruthfulQA/HaluEval $\Delta$ and lower Brier $\Delta$ indicate less degradation. \ourmethod\ achieves competitive task accuracy with substantially less degradation across all three axes.}
\label{fig:scatter_factuality}
\end{figure}
\subsection{Factuality and Truthfulness}
\label{subsec:factualityandtruthfulness}
Fine-tuning often degrades factuality and increases hallucination, especially when adapting models to new knowledge~\citep{gekhman2024does, zucchet2025language}. We therefore evaluate whether \ourmethod\ preserves factuality while still improving target-task performance. We use the multiple-choice version of TruthfulQA~\citep{lin2022truthfulqa, eval-harness}, which measures whether a model prefers truthful answers over common false or misleading alternatives. We report the change in TruthfulQA accuracy relative to the pretrained model, so higher values indicate better preservation of factuality.
\begin{table}[h]
\centering
\caption{
Effect of fine-tuning methods on truthfulness, hallucination and calibration(Qwen3-4B). Brier $\Delta$ is relative to the pretrained model (lower is better); TruthfulQA and HaluEval are relative to the pretrained model (higher is better).
}
\label{tab:hallucination_tofu}
\resizebox{0.9\textwidth}{!}{
\begin{tabular}{ll|c|cc|c}
\toprule
\textbf{Model} & \textbf{Method} & \textbf{Task Acc.} $\uparrow$ & \textbf{TruthfulQA $\Delta$} $\uparrow$ & \textbf{HaluEval $\Delta$} $\uparrow$ & \textbf{Brier $\Delta$} $\downarrow$ \\
\midrule
\multirow{10}{*}{\rotatebox[origin=c]{90}{Qwen3-4B}}
 & Base                                      & $53.5$ & $+0.0$   & $+0.0$   & $+0.0$  \\
 & SFT                                & $83.5$ & $-13.9$  & $-9.8$   & $+14.7$ \\
 & FLOW~\citep{sanyal2025upweighting}        & $77.5$ & $-11.7$  & $-8.2$   & $+12.8$ \\
 & DFT~\citep{wu2025generalization}          & $55.8$ & $-4.9$   & $+2.7$   & $+6.7$  \\
 & TALR~\citep{lin2025sft}                   & $62.5$ & $-5.4$   & $+5.1$   & $+12.7$ \\
 & STM~\citep{wu2025mitigating}              & $53.8$ & $-0.7$   & $+3.7$   & $+1.9$  \\
 & L2 Reg~\citep{kirkpatrick2017overcoming}  & $83.5$ & $-12.0$  & $-1.2$   & $+13.6$ \\
 & WiSE-FT~\citep{wortsman2022robust}        & $82.5$ & $-19.9$  & $-7.8$   & $+17.5$ \\
 & LoRA~\citep{hu2022lora}                   & $82.3$ & $-16.4$  & $-4.3$   & $+14.9$ \\
\cmidrule{2-6}
\rowcolor{red!15} & \ourmethod\              & $83.3$ & $-2.6$   & $+3.3$   & $+9.1$  \\
\bottomrule
\end{tabular}
}
\end{table}

Figure~\ref{fig:scatter_factuality} shows the trade-off between target-task performance and TruthfulQA degradation for KA. The corresponding results for Science and LA are provided in Appendix~\ref{app:experiments} (Figures~\ref{fig:scatter_factuality_science} and~\ref{fig:scatter_factuality_galician}; Tables~\ref{tab:hallucination_science} and~\ref{tab:hallucination_galician}). We report the detailed table for KA in Table~\ref{tab:hallucination_tofu}, since factuality degradation is largest in this setting and prior work suggests that hallucination is especially affected when fine-tuning on new knowledge. Across the three tasks on Qwen3-4B, \ourmethod\ gives a strong trade-off between target-task performance and factuality preservation: it achieves competitive target-task performance while substantially reducing TruthfulQA degradation. In KA, standard SFT reaches $83.5\%$ task accuracy but reduces TruthfulQA by $13.9$ points, whereas \ourmethod\ reaches $83.3\%$ accuracy with only a $2.6$ point drop. This suggests that methods that preserve general reasoning and instruction-following performance may also help preserve factuality.

\subsection{Hallucination Detection}
\label{subsec:hallucination}
LLMs can often recognize hallucinated responses when explicitly asked to compare them against factual alternatives~\citep{li2023halueval}. We ask whether this hallucination-detection ability is preserved after fine-tuning, and whether \ourmethod\ better maintains it while still improving target-task performance. We evaluate using HaluEval, where the model is given a question, a correct response, and a hallucinated response, and must identify which response is hallucinated. We report the change in accuracy relative to the pretrained model, so higher values indicate better preservation of hallucination detection.

Figure~\ref{fig:scatter_halu} reports the HaluEval results; detailed tables for the three tasks are provided in Table~\ref{tab:hallucination_tofu} and Appendix~\ref{app:experiments}. In KA, standard SFT reaches $83.5\%$ task accuracy but reduces HaluEval by $9.8$ points, whereas \ourmethod\ reaches $83.3\%$ accuracy and improves HaluEval by $3.3$ points. Interestingly, unlike the general benchmark results, several baselines also improve HaluEval in KA. This is most visible for the token- or sequence-reweighting methods, such as DFT, TALR, and STM, which suppress high-loss examples and therefore reduce forgetting more directly. However, these methods achieve substantially lower KA accuracy, while \ourmethod\ preserves hallucination detection without sacrificing target-task performance. The same trend appears only partially in Science and LA, suggesting that hallucination-detection changes depend on both the fine-tuning method and the target dataset. Understanding this interaction is an interesting direction for future work.

\subsection{Confidence Calibration}
\label{subsec:calibration}

Calibration measures whether a model's stated confidence corresponds to its correctness: a reliable model should not be highly confident when it is wrong. Prior work has shown that fine-tuning can degrade calibration~\citep{he2023preserving}. We therefore evaluate whether \ourmethod\ preserves calibration after adaptation. Following prior work, we elicit verbalized confidence by prompting the model to report a confidence score from $0$ to $100$~\citep{kadavath2022language, wei2024measuring}. We measure calibration on TruthfulQA using the Brier score between this stated confidence and whether the model's answer is correct. We report the change relative to the pretrained model; lower values indicate better calibration preservation.

Figure~\ref{fig:scatter_brier} reports the calibration results, and Table~\ref{tab:hallucination_tofu} gives the corresponding values. Fine-tuning generally increases the Brier score across methods, meaning that models become less calibrated after adaptation. This degradation is especially large for standard SFT. In KA, for example, standard SFT increases the Brier score by $14.7$ points. \ourmethod\ reduces this increase to $9.1$ points while maintaining competitive target-task accuracy. Some reweighting methods reduce the Brier score increase further, but they also obtain much lower target-task accuracy, giving a worse trade-off. Thus, \ourmethod\ gives a favorable trade-off between target-task performance and calibration preservation. At the same time, the remaining Brier score increase shows that calibration is not fully preserved, leaving substantial room for future improvement.

%% file: conclusion.tex
\section{Conclusion, Limitations, and Future Work}
\label{sec:conclusion}

We introduce \ourmethod, an adaptive learning rate schedule that mitigates catastrophic forgetting during fine-tuning. Our theoretical analysis shows that per-step forgetting is bounded by the learning rate times the square root of the current mini-batch loss, which motivates reducing the learning rate on high-loss batches and relaxing it as the model converges. Empirically, FINCH achieves a 93\% average reduction in forgetting while matching standard SFT on target-task accuracy, and gives a Pareto-optimal trade-off across factuality, hallucination detection, and confidence calibration.

\paragraph{Limitations and Future Work.}
Our results also point to several directions for future work. Our analysis is stated for SGD, while our experiments use AdamW; closing this gap is an interesting direction for future work. Additionally, while FINCH substantially reduces calibration degradation relative to SFT, it does not fully preserve the calibration of the pretrained model. Developing methods that close this remaining calibration gap after fine-tuning is an important open problem. Finally, due to computational constraints, our experiments are limited to models up to 8B parameters, and evaluating FINCH at larger scales is a natural next step.

%% file: appendix.tex
\newpage
\appendix
\section{Proofs}
\label{app:proofs}

\subsection{Auxiliary bounds implied by Assumption~\ref{asm:bounded-network}}

We first record standard consequences of Assumption~\ref{asm:bounded-network}. Since the input domain is bounded, the fine-tuning trajectory remains in a bounded parameter region, and the activation functions have bounded first and second derivatives on the relevant range, the network computation is uniformly controlled along training. In particular, the logits, the logit Jacobian, the score-function gradient, and the log-probability Hessian are uniformly bounded for all contexts and labels, uniformly over the bounded parameter region traversed by the fine-tuning trajectory. Similar boundedness and smoothness assumptions are standard in learning-theoretic analyses of gradient methods~\citep{shalev2014understanding}, and related work derives explicit Lipschitz or derivative bounds for neural networks~\citep{virmaux2018lipschitz, herrera2020local}.

\begin{lemma}
\label{lem:bounded-jacobian}
Under Assumption~\ref{asm:bounded-network}, there exist constants $M,G,H>0$ such that for all parameters $\theta$ in the bounded region traversed during fine-tuning, contexts $x$, and labels $y$,
\[
\|J_\theta(x)\|_{\mathrm{op}} \le M,
\qquad
\|\nabla_\theta \log p_\theta(y\mid x)\|_2 \le G,
\qquad
\|\nabla_\theta^2 \log p_\theta(y\mid x)\|_{\mathrm{op}} \le H,
\]
where $J_\theta(x)$ denotes the Jacobian of the logits with respect to $\theta$.
\end{lemma}

\begin{proof}
We give the argument for a finite-depth differentiable network; the same reasoning applies to transformer architectures under the corresponding bounded-computation assumption, since they are finite computational graphs composed of differentiable operations.

For notational simplicity, consider a layerwise representation
\[
h_0=x,
\qquad
a_\ell = W_\ell h_{\ell-1}+b_\ell,
\qquad
h_\ell=\sigma_\ell(a_\ell),
\qquad
\ell=1,\dots,L,
\]
with final logits $z_\theta(x)=h_L$.
The argument below only uses finiteness of the computation graph and boundedness of the quantities appearing in the chain rule, not the specific feedforward form above.

We first show that all hidden states are uniformly bounded. Since the input domain is bounded, $\|h_0\|_2=\|x\|_2\le B_x$. Assume inductively that $h_{\ell-1}$ is uniformly bounded. Since $\theta$ lies in the bounded parameter region considered in Assumption~\ref{asm:bounded-network}, the weights $W_\ell$ and biases $b_\ell$ are uniformly bounded. Therefore the pre-activation
\[
a_\ell = W_\ell h_{\ell-1}+b_\ell
\]
is uniformly bounded. Since $\sigma_\ell$ has bounded first derivative, it is Lipschitz; because $\sigma_\ell$ is finite at the origin, it maps bounded sets to bounded sets. Hence $h_\ell=\sigma_\ell(a_\ell)$ is uniformly bounded. By induction, every hidden state, and therefore the final logits $z_\theta(x)$, are uniformly bounded.

Next consider the logit Jacobian $J_\theta(x)=\nabla_\theta z_\theta(x)$. Each entry of this Jacobian is obtained by repeated application of the chain rule through the network. Every factor that appears is either an input coordinate, a hidden activation, a parameter, or a first derivative of an activation. All such factors are uniformly bounded by the previous paragraph and Assumption~\ref{asm:bounded-network}. Since the network has finite depth and finite parameter dimension, there exists a constant $M>0$ such that
\[
\|J_\theta(x)\|_{\mathrm{op}} \le M.
\]

For the score-function gradient,
\[
\nabla_\theta \log p_\theta(y\mid x)
=
J_\theta(x)^\top \nabla_z \log p_\theta(y\mid x).
\]
For a softmax output layer,
\[
\nabla_z \log p_\theta(y\mid x)=e_y-p_\theta(\cdot\mid x),
\]
so
\[
\|\nabla_z \log p_\theta(y\mid x)\|_2 \le \sqrt{2}.
\]
Hence
\[
\|\nabla_\theta \log p_\theta(y\mid x)\|_2
\le
\|J_\theta(x)\|_{\mathrm{op}}\,
\|\nabla_z \log p_\theta(y\mid x)\|_2
\le
\sqrt{2}\,M.
\]
Absorbing $\sqrt{2}$ into the constant gives the desired bound with some $G>0$.

Finally, consider the Hessian $\nabla_\theta^2 \log p_\theta(y\mid x)$. Differentiating once more introduces finitely many terms produced by the product and chain rules. These terms involve bounded inputs, bounded hidden activations, bounded parameters, bounded first derivatives of activations, bounded second derivatives of activations, and first- and second-order derivatives of the log-softmax map. Since the logits are uniformly bounded and the log-softmax map is smooth, these log-softmax derivatives are also uniformly bounded on the relevant range. Therefore, since the architecture has finite depth and finite parameter dimension, all entries of the Hessian are uniformly bounded. Since the parameter dimension is finite, this implies a uniform operator-norm bound. Hence
\[
\|\nabla_\theta^2 \log p_\theta(y\mid x)\|_{\mathrm{op}} \le H
\]
for some constant $H>0$.
\end{proof}

\subsection{Proof of Theorem~\ref{thm:stepwise-bound}}

We now prove the per-step forgetting bound for the sequence-level mini-batch loss used during training.
\renewcommand{\thetheorem}{\ref{thm:stepwise-bound}}
\begin{theorem}[Per-step forgetting bound]
Suppose each training example in the mini-batch $\mathcal B_i$ is a sequence
$s=(x_{1:T_s},y_{1:T_s})$, and define its average token-level cross-entropy as
\[
\ell_s(\theta)
:=
\frac{1}{T_s}\sum_{t=1}^{T_s}
\CE\bigl(q_t(\cdot\mid x_{<t})\|p_\theta(\cdot\mid x_{<t})\bigr),
\]
where $q_t(\cdot\mid x_{<t})$ is the target distribution at position $t$ conditioned on the prefix $x_{<t}$. Let
\[
\mathcal{L}_{\mathcal B_i}(\theta)
:=
\frac{1}{|\mathcal B_i|}\sum_{s\in\mathcal B_i}\ell_s(\theta)
\]
denote the average sequence loss on the mini-batch used at training step $i$. Under Assumption~\ref{asm:bounded-network}, there exist constants $C_1,C_2>0$ such that
\[
\Delta\mathcal{L}_{\mathrm{old}}(p_i)
\le
C_1 \eta_i \sqrt{\mathcal{L}_{\mathcal B_i}(\theta_i)}
+
C_2 \eta_i^2 \mathcal{L}_{\mathcal B_i}(\theta_i).
\]
Consequently, for sufficiently small $\eta_i$,
\[
\Delta\mathcal{L}_{\mathrm{old}}(p_i)
=
O\!\left(\eta_i\,\sqrt{\mathcal{L}_{\mathcal B_i}(\theta_i)}\right).
\]
\end{theorem}

\begin{proof}
Let
\[
g_i^{\mathcal B}
:=
\nabla_\theta \mathcal{L}_{\mathcal B_i}(\theta_i)
\]
be the mini-batch gradient, so the SGD update is
\[
\theta_{i+1}-\theta_i = -\eta_i g_i^{\mathcal B}.
\]

We begin from the old-task loss:
\[
\mathcal{L}_{\mathrm{old}}(p_{i+1})-\mathcal{L}_{\mathrm{old}}(p_i)
=
-\mathbb{E}_{(x_{\mathrm{old}},y_{\mathrm{old}})\sim q_{\mathrm{old}}}
\left[
\log\frac{p_{i+1}(y_{\mathrm{old}}\mid x_{\mathrm{old}})}
{p_i(y_{\mathrm{old}}\mid x_{\mathrm{old}})}
\right].
\]
Therefore,
\[
\Delta\mathcal{L}_{\mathrm{old}}(p_i)
\le
\sup_{(x_{\mathrm{old}},y_{\mathrm{old}})}
\left|
\log\frac{p_{i+1}(y_{\mathrm{old}}\mid x_{\mathrm{old}})}
{p_i(y_{\mathrm{old}}\mid x_{\mathrm{old}})}
\right|.
\]
It thus suffices to bound the pointwise log-ratio on the right-hand side.

Fix any $(x_{\mathrm{old}},y_{\mathrm{old}})$ and define
\[
f(\theta)=\log p_\theta(y_{\mathrm{old}}\mid x_{\mathrm{old}}).
\]
By Taylor's theorem, for some point $\tilde\theta$ on the line segment joining $\theta_i$ and $\theta_{i+1}$,
\[
f(\theta_{i+1})-f(\theta_i)
=
\nabla_\theta f(\theta_i)^\top(\theta_{i+1}-\theta_i)
+
\frac{1}{2}
(\theta_{i+1}-\theta_i)^\top
\nabla_\theta^2 f(\tilde\theta)
(\theta_{i+1}-\theta_i).
\]
We take the bounded parameter region in Assumption~\ref{asm:bounded-network} to contain the line segments between consecutive iterates, so Lemma~\ref{lem:bounded-jacobian} also applies at $\tilde\theta$.
Substituting the update rule,
\[
\log\frac{p_{i+1}(y_{\mathrm{old}}\mid x_{\mathrm{old}})}
{p_i(y_{\mathrm{old}}\mid x_{\mathrm{old}})}
=
-\eta_i \nabla_\theta \log p_{\theta_i}(y_{\mathrm{old}}\mid x_{\mathrm{old}})^\top g_i^{\mathcal B}
+
\frac{\eta_i^2}{2}
(g_i^{\mathcal B})^\top
\nabla_\theta^2 \log p_{\tilde\theta}(y_{\mathrm{old}}\mid x_{\mathrm{old}})
g_i^{\mathcal B}.
\]
Taking absolute values and using Cauchy--Schwarz gives
\[
\left|
\log\frac{p_{i+1}(y_{\mathrm{old}}\mid x_{\mathrm{old}})}
{p_i(y_{\mathrm{old}}\mid x_{\mathrm{old}})}
\right|
\le
\eta_i
\|\nabla_\theta \log p_{\theta_i}(y_{\mathrm{old}}\mid x_{\mathrm{old}})\|_2
\|g_i^{\mathcal B}\|_2
+
\frac{\eta_i^2}{2}
\|\nabla_\theta^2 \log p_{\tilde\theta}(y_{\mathrm{old}}\mid x_{\mathrm{old}})\|_{\mathrm{op}}
\|g_i^{\mathcal B}\|_2^2.
\]
By Lemma~\ref{lem:bounded-jacobian}, it remains to bound $\|g_i^{\mathcal B}\|_2$ in terms of the mini-batch loss.
We first bound the gradient of the loss for a single sequence $s=(x_{1:T_s},y_{1:T_s})$. By definition,
\[
\ell_s(\theta)
=
\frac{1}{T_s}\sum_{t=1}^{T_s}
\CE\bigl(q_t(\cdot\mid x_{<t})\|p_\theta(\cdot\mid x_{<t})\bigr),
\]
hence
\[
\nabla_\theta \ell_s(\theta_i)
=
\frac{1}{T_s}\sum_{t=1}^{T_s}
\nabla_\theta
\CE\bigl(q_t(\cdot\mid x_{<t})\|p_{\theta_i}(\cdot\mid x_{<t})\bigr).
\]
For each token position $t$, the gradient of the token-level cross-entropy is
\[
\nabla_\theta
\CE\bigl(q_t(\cdot\mid x_{<t})\|p_{\theta_i}(\cdot\mid x_{<t})\bigr)
=
J_{\theta_i}(x_{<t})^\top
\bigl(p_{\theta_i}(\cdot\mid x_{<t})-q_t(\cdot\mid x_{<t})\bigr),
\]
where $J_{\theta_i}(x_{<t})$ is the Jacobian of the logits at context $x_{<t}$ with respect to $\theta$. Therefore,
\[
\left\|
\nabla_\theta
\CE\bigl(q_t(\cdot\mid x_{<t})\|p_{\theta_i}(\cdot\mid x_{<t})\bigr)
\right\|_2
\le
\|J_{\theta_i}(x_{<t})\|_{\mathrm{op}}
\,
\|p_{\theta_i}(\cdot\mid x_{<t})-q_t(\cdot\mid x_{<t})\|_2.
\]
Using $\|v\|_2\le \|v\|_1$, Pinsker's inequality, and
\[
\KL\bigl(q_t(\cdot\mid x_{<t})\|p_{\theta_i}(\cdot\mid x_{<t})\bigr)
\le
\CE\bigl(q_t(\cdot\mid x_{<t})\|p_{\theta_i}(\cdot\mid x_{<t})\bigr),
\]
we get
\[
\left\|
\nabla_\theta
\CE\bigl(q_t(\cdot\mid x_{<t})\|p_{\theta_i}(\cdot\mid x_{<t})\bigr)
\right\|_2
\le
\|J_{\theta_i}(x_{<t})\|_{\mathrm{op}}
\sqrt{
2\,
\CE\bigl(q_t(\cdot\mid x_{<t})\|p_{\theta_i}(\cdot\mid x_{<t})\bigr)
}.
\]
By Lemma~\ref{lem:bounded-jacobian}, $\|J_{\theta_i}(x_{<t})\|_{\mathrm{op}}\le M$, so
\[
\left\|
\nabla_\theta
\CE\bigl(q_t(\cdot\mid x_{<t})\|p_{\theta_i}(\cdot\mid x_{<t})\bigr)
\right\|_2
\le
M\sqrt{
2\,
\CE\bigl(q_t(\cdot\mid x_{<t})\|p_{\theta_i}(\cdot\mid x_{<t})\bigr)
}.
\]
Now apply the triangle inequality across token positions:
\[
\|\nabla_\theta \ell_s(\theta_i)\|_2
\le
\frac{1}{T_s}\sum_{t=1}^{T_s}
\left\|
\nabla_\theta
\CE\bigl(q_t(\cdot\mid x_{<t})\|p_{\theta_i}(\cdot\mid x_{<t})\bigr)
\right\|_2.
\]
Combining with the previous bound,
\[
\|\nabla_\theta \ell_s(\theta_i)\|_2
\le
\frac{M}{T_s}\sum_{t=1}^{T_s}
\sqrt{
2\,
\CE\bigl(q_t(\cdot\mid x_{<t})\|p_{\theta_i}(\cdot\mid x_{<t})\bigr)
}.
\]
Using Jensen's inequality for the concave function $u\mapsto \sqrt{u}$,
\[
\frac{1}{T_s}\sum_{t=1}^{T_s}
\sqrt{
2\,
\CE\bigl(q_t(\cdot\mid x_{<t})\|p_{\theta_i}(\cdot\mid x_{<t})\bigr)
}
\le
\sqrt{
\frac{2}{T_s}\sum_{t=1}^{T_s}
\CE\bigl(q_t(\cdot\mid x_{<t})\|p_{\theta_i}(\cdot\mid x_{<t})\bigr)
}.
\]
Thus
\[
\|\nabla_\theta \ell_s(\theta_i)\|_2
\le
M\sqrt{2\,\ell_s(\theta_i)}.
\]
We now lift this bound from a single sequence to the whole mini-batch. Since
\[
\mathcal{L}_{\mathcal B_i}(\theta)
=
\frac{1}{|\mathcal B_i|}\sum_{s\in\mathcal B_i}\ell_s(\theta),
\]
its gradient is
\[
g_i^{\mathcal B}
=
\nabla_\theta \mathcal{L}_{\mathcal B_i}(\theta_i)
=
\frac{1}{|\mathcal B_i|}\sum_{s\in\mathcal B_i}\nabla_\theta \ell_s(\theta_i).
\]
Applying the triangle inequality across sequences,
\[
\|g_i^{\mathcal B}\|_2
\le
\frac{1}{|\mathcal B_i|}\sum_{s\in\mathcal B_i}
\|\nabla_\theta \ell_s(\theta_i)\|_2
\le
\frac{M}{|\mathcal B_i|}\sum_{s\in\mathcal B_i}\sqrt{2\,\ell_s(\theta_i)}.
\]
Applying Jensen's inequality once more,
\[
\frac{1}{|\mathcal B_i|}\sum_{s\in\mathcal B_i}\sqrt{2\,\ell_s(\theta_i)}
\le
\sqrt{
\frac{2}{|\mathcal B_i|}\sum_{s\in\mathcal B_i}\ell_s(\theta_i)
}
=
\sqrt{2\,\mathcal{L}_{\mathcal B_i}(\theta_i)}.
\]
Therefore,
\[
\|g_i^{\mathcal B}\|_2
\le
M\sqrt{2\,\mathcal{L}_{\mathcal B_i}(\theta_i)}.
\]

Substituting this bound into the Taylor estimate and using Lemma~\ref{lem:bounded-jacobian} again, we obtain
\[
\left|
\log\frac{p_{i+1}(y_{\mathrm{old}}\mid x_{\mathrm{old}})}
{p_i(y_{\mathrm{old}}\mid x_{\mathrm{old}})}
\right|
\le
C_1\,\eta_i\,\sqrt{\mathcal{L}_{\mathcal B_i}(\theta_i)}
+
C_2\,\eta_i^2\,\mathcal{L}_{\mathcal B_i}(\theta_i),
\]
for constants $C_1,C_2>0$ independent of $i$, $(x_{\mathrm{old}},y_{\mathrm{old}})$, and the mini-batch.

Taking the supremum over $(x_{\mathrm{old}},y_{\mathrm{old}})$ gives
\[
\Delta\mathcal{L}_{\mathrm{old}}(p_i)
\le
C_1\,\eta_i\,\sqrt{\mathcal{L}_{\mathcal B_i}(\theta_i)}
+
C_2\,\eta_i^2\,\mathcal{L}_{\mathcal B_i}(\theta_i).
\]

It remains only to justify the simplified big-$O$ form. Assumption~\ref{asm:bounded-network} implies that the logits are uniformly bounded, hence the softmax probabilities are bounded away from zero. Therefore the token-level cross-entropies, sequence losses, and mini-batch losses are uniformly bounded above: there exists $L_{\max}<\infty$ such that
\[
\mathcal{L}_{\mathcal B_i}(\theta_i)\le L_{\max}
\]
for all $i$. Hence
\[
\eta_i^2\mathcal{L}_{\mathcal B_i}(\theta_i)
=
\eta_i\sqrt{\mathcal{L}_{\mathcal B_i}(\theta_i)}
\cdot
\eta_i\sqrt{\mathcal{L}_{\mathcal B_i}(\theta_i)}
\le
\eta_i\sqrt{\mathcal{L}_{\mathcal B_i}(\theta_i)}
\cdot
\eta_i\sqrt{L_{\max}}.
\]
Thus, for sufficiently small $\eta_i$, the second-order term is absorbed into the first-order term, and
\[
\Delta\mathcal{L}_{\mathrm{old}}(p_i)
=
O\!\left(\eta_i\,\sqrt{\mathcal{L}_{\mathcal B_i}(\theta_i)}\right).
\]
\end{proof}
\renewcommand{\thetheorem}{\arabic{theorem}}

\subsection{Proof of Corollary~\ref{cor:cumulative}}

\renewcommand{\thecorollary}{\ref{cor:cumulative}}
\begin{corollary}
Under the conditions of Theorem~\ref{thm:stepwise-bound}, suppose the learning rates are sufficiently small and
\[
\eta_i = \frac{\kappa}{\sqrt{\mathcal{L}_{\mathcal B_i}(\theta_i)}}.
\]
Then the cumulative forgetting satisfies
\[
\mathcal{L}_{\mathrm{old}}(p_T)-\mathcal{L}_{\mathrm{old}}(p_0)
=
O(T\kappa).
\]
\end{corollary}

\begin{proof}
Since the learning rates are sufficiently small, Theorem~\ref{thm:stepwise-bound} gives
\[
\Delta\mathcal{L}_{\mathrm{old}}(p_i)
=
O\!\left(\eta_i\sqrt{\mathcal{L}_{\mathcal B_i}(\theta_i)}\right).
\]
Substituting
\[
\eta_i = \frac{\kappa}{\sqrt{\mathcal{L}_{\mathcal B_i}(\theta_i)}}
\]
gives
\[
\Delta\mathcal{L}_{\mathrm{old}}(p_i)=O(\kappa)
\]
for each step $i$. Summing over $i=0,\dots,T-1$,
\[
\mathcal{L}_{\mathrm{old}}(p_T)-\mathcal{L}_{\mathrm{old}}(p_0)
=
\sum_{i=0}^{T-1}\Delta\mathcal{L}_{\mathrm{old}}(p_i)
=
O(T\kappa).
\]
\end{proof}
\renewcommand{\thecorollary}{\arabic{corollary}}

\section{Experimental Details}
\label{app:experiments}
\subsection{Tasks}
\label{app:experiments-tasks}

We evaluate \ourmethod\ on three settings where the pretrained model does not have good performance and a large fraction of task-relevant tokens are hard:

\begin{itemize}
    \item \textit{Language adaptation (LA):} Instruction following in Galician, a low-resource language with limited pretraining coverage. We use Galician Alpaca, a translated version of the Stanford Alpaca instruction-following dataset~\citep{alpaca, chen2024monolingual}. We translate instructions and outputs separately using GPT-5.2 with the following prompt:
    \begin{quote}
    "Translate the following text from English to Galician.\\
    Requirements:
    \begin{itemize}
        \item Keep the exact meaning.
        \item Keep formatting and punctuation where possible.
        \item Return only the translated text.
    \end{itemize}
    Text: \{text\}"
    \end{quote}
    We use a 70/5/25 train/validation/test split.

    \item \textit{Science.} Undergraduate-level scientific reasoning using the Chemistry
        L-3 subset of SciKnowEval~\citep{feng2024sciknoweval}, a multiple-choice benchmark
        covering chemistry concepts. Solutions were obtained from Shenfeld et al.~\cite{shenfeld2026self} that used gpt-4o to generate solutions. We use a 70/5/25 train/validation/test split.

    \item \textit{Knowledge acquisition (KA).} Acquiring novel factual information using
    TOFU~\citep{maini2024tofu}, a dataset of 200 synthetic author profiles each consisting
    of 20 question--answer pairs probing biographical facts. We train and validate on a 90/10 split of the
    full set of 4{,}000 question--answer pairs. For evaluation, we use a subset of 400
    paraphrased questions paired with both a correct answer and several perturbed
    (incorrect) answers, which allows us to compute multiple-choice accuracy.

\end{itemize}

\subsection{Hyperparameters}
All models are optimized with AdamW. For all baselines, we use a cosine learning rate schedule with a linear warmup over the first 5\% of training steps, train for up to 5 epochs, and select the checkpoint with the best validation performance on the target task. We clip gradients to a maximum norm of 1.0 for all methods (except for the small grad norm baseline for which we use 0.1 as the maximum norm).

We sweep over learning rates in $\{5\times10^{-6},\ 1\times10^{-5},\ 2\times10^{-5},\ 3\times10^{-5},\ 5\times10^{-5},\ 1\times10^{-4}\}$ and batch sizes in $\{16, 32\}$ for all methods except LoRA, for which we follow Sanyal et al.~\cite{sanyal2025upweighting} and sweep over $\{2\times10^{-4},\ 2\times10^{-3},\ 2\times10^{-2},\ 2\times10^{-1}\}$.

For the method-specific hyperparameters, we follow Sanyal et al.~\cite{sanyal2025upweighting} when applicable. For WiSE-FT~\citep{wortsman2022robust}, we set the interpolation coefficient to $\alpha=0.5$, and for L2 regularization~\citep{kirkpatrick2017overcoming}, we set $\lambda=10^{-3}$. For STM~\citep{wu2025mitigating}, we use threshold $2.5$; larger thresholds improve target-task accuracy but lead to forgetting comparable to standard SFT. For SFT (small lr)~\citep{lin2025sft}, the original paper does not specify a general rule for choosing the small learning rate, so we choose the largest learning rate whose average forgetting is less than $2$ percentage points on the validation benchmarks, giving the strongest target-task performance under a low-forgetting constraint.

All experiments were performed on single H200.
\subsection{Evaluation Metrics}
\label{app:evaluation-metrics}

\paragraph{Average old-task change.}
For each method, Avg.~$\Delta$ is the mean change in performance relative to the pretrained model across HellaSwag, WinoGrande, IFEval, and MMLU:
\[
\mathrm{Avg.}~\Delta
=
\frac{1}{4}
\sum_{b \in \{\mathrm{HS}, \mathrm{WG}, \mathrm{IFEval}, \mathrm{MMLU}\}}
\left(\mathrm{Acc}_{b}^{\mathrm{ft}}-\mathrm{Acc}_{b}^{\mathrm{base}}\right).
\]
Higher values indicate better preservation of general capabilities.

\paragraph{Brier score.}
For calibration, we compute the Brier score on TruthfulQA using the model's verbalized confidence:
\[
\mathrm{Brier}
=
\frac{1}{N}
\sum_{i=1}^{N}
(\hat p_i - y_i)^2,
\]
where $\hat p_i \in [0,1]$ is the model's reported confidence in its answer and $y_i \in \{0,1\}$ indicates whether the answer is correct. We report $\Delta$ Brier relative to the pretrained model, so lower values indicate better calibration preservation.
\subsection{Evaluation Prompts}
\label{app:evaluation-prompts}

\paragraph{Galician win-tie judge prompt.}
For low-resource language adaptation, we evaluate Galician language quality using pairwise head-to-head judging. For each example, we randomly swap the order of the two model responses to reduce positional bias. The judge is instructed to compare only the Galician language quality of the two responses, ignoring factual correctness, helpfulness, and content quality. The exact judge instruction is:

\begin{tcolorbox}[
  colback=gray!5,
  colframe=gray!60,
  title={Galician Language Quality Judge Prompt},
  fonttitle=\bfseries,
  breakable
]
\small
Please act as an impartial judge and evaluate the Galician language quality of the responses provided by two AI assistants to the user prompt below. You will be given assistant A's answer and assistant B's answer.

Evaluate each response on the following criteria:
\begin{enumerate}
    \item Is the response written in Galician (not Spanish, Portuguese, English, or other languages)?
    \item How natural and fluent is the Galician? Does it sound like a native speaker, or does it have Spanish/Portuguese interference?
    \item How consistent is the Galician throughout the response---does it code-switch mid-response?
\end{enumerate}

After your brief explanation, you must output only one of the following choices as your final verdict with a label:
\begin{enumerate}
    \item Assistant A is significantly better: \texttt{[[A>{}>B]]}
    \item Assistant A is slightly better: \texttt{[[A>B]]}
    \item Tie, relatively the same: \texttt{[[A=B]]}
    \item Assistant B is slightly better: \texttt{[[B>A]]}
    \item Assistant B is significantly better: \texttt{[[B>{}>A]]}
\end{enumerate}

Example output: ``My final verdict is tie: \texttt{[[A=B]]}''.
\end{tcolorbox}

\paragraph{TruthfulQA confidence prompt.}
For calibration, we evaluate TruthfulQA with verbalized confidence. The model is asked to provide both an answer and a confidence score, which is then used to compute the Brier score. The exact confidence prompt template is:

\begin{tcolorbox}[
  colback=gray!5,
  colframe=gray!60,
  title={TruthfulQA Confidence Prompt},
  fonttitle=\bfseries,
  breakable
]
\small
\begin{verbatim}
Answer this multiple-choice question.
Return only one letter: A, B, C, or D, and a confidence score between 0

Question: {question}

Options:
A. {choice_a}
B. {choice_b}
C. {choice_c}
D. {choice_d}

Answer:
\end{verbatim}
\end{tcolorbox}
\subsection{What if we have a small maximum gradient norm?}
\label{app:grad-norm}
We also test whether forgetting can be mitigated by simply reducing the maximum gradient norm during fine-tuning. Table~\ref{tab:small_grad_norm} reports this control on knowledge acquisition.
\begin{table}[h]
\centering
\caption{
Effect of reducing the maximum gradient norm on knowledge acquisition. Task Acc. reports held-out multiple-choice accuracy on author-profile questions, and Avg. $\Delta$ is the average change on old-task benchmarks.
}
\label{tab:small_grad_norm}
\resizebox{0.65\textwidth}{!}{
\begin{tabular}{ll|cc}
\toprule
\textbf{Model} & \textbf{Max grad norm} & \textbf{Task Acc.} $\uparrow$ & \textbf{Avg. $\Delta$} $\uparrow$ \\
\midrule
\multirow{5}{*}{\rotatebox[origin=c]{90}{Qwen3-4B}}
 & $1$      & $83.5$ & $-9.6$ \\
 & $0.1$    & $82.3$ & $-8.4$ \\
 & $0.01$   & $81.8$ & $-8.3$ \\
 & $0.001$  & $71.8$ & $-7.1$ \\
 & $0.0001$ & $58.5$ & $-4.6$ \\
\midrule
\multirow{5}{*}{\rotatebox[origin=c]{90}{Llama-3-8B}}
 & $1$      & $83.3$ & $-9.0$ \\
 & $0.1$    & $83.0$ & $-10.1$ \\
 & $0.01$   & $82.5$ & $-8.5$ \\
 & $0.001$  & $78.0$ & $-3.6$ \\
 & $0.0001$ & $62.8$ & $-0.4$ \\
\bottomrule
\end{tabular}
}
\end{table}
Reducing the clipping threshold gives only a weak trade-off. Moderate clipping thresholds, such as $0.1$ or $0.01$, preserve target-task accuracy but barely improve forgetting; stronger clipping improves Avg. $\Delta$ mainly by substantially hurting target-task accuracy. Thus, small gradient clipping does not help.

\subsection{Additional Results}

\subsubsection{Additional Results for Section~\ref{sec:experiments-preserves}}
Table~\ref{tab:results_science} and Table~\ref{tab:results_galician} provide the results for Science and LA respectively.
\begin{table}[h]
\centering
\caption{
Target-task adaptation and forgetting on Science QA. Task Acc. reports held-out multiple-choice accuracy on science questions. HellaSwag (HS), WinoGrande (WG), IFEval, and MMLU report absolute performance changes relative to the pretrained model, and Avg. $\Delta$ is their mean. Higher values indicate better preservation of general capabilities.
}
\label{tab:results_science}
\resizebox{\textwidth}{!}{
\begin{tabular}{ll|c|cccc|c}
\toprule
\textbf{Model} & \textbf{Method} & \textbf{Task Acc.} $\uparrow$ & \textbf{HS} $\uparrow$ & \textbf{WG} $\uparrow$ & \textbf{IFEval} $\uparrow$ & \textbf{MMLU} $\uparrow$ & \textbf{Avg. $\Delta$} $\uparrow$ \\
\midrule
\multirow{11}{*}{\rotatebox[origin=c]{90}{Qwen3-4B}}
 & Base                                      & $29.2$ & $+0.0$  & $+0.0$  & $+0.0$   & $+0.0$  & $+0.0$  \\
 & SFT                                & $62.5$ & $-11.2$ & $+2.5$  & $-17.6$  & $-17.3$ & $-10.9$ \\
 & SFT (small lr)                            & $56.7$ & $-5.3$  & $-0.4$  & $+4.9$   & $+3.1$  & $+0.6$  \\
 & FLOW~\citep{sanyal2025upweighting}        & $58.3$ & $-7.0$  & $-3.0$  & $-6.4$   & $-10.0$ & $-6.6$  \\
 & DFT~\citep{wu2025generalization}          & $46.7$ & $-3.4$  & $-0.3$  & $-5.7$   & $+1.4$  & $-2.0$  \\
 & TALR~\citep{lin2025sft}                   & $48.3$ & $-0.8$  & $-0.3$  & $-2.5$   & $+3.8$  & $+0.1$  \\
 & STM~\citep{wu2025mitigating}              & $49.2$ & $-2.4$  & $+2.9$  & $-1.7$   & $+2.6$  & $+0.4$  \\
 & L2 Reg~\citep{kirkpatrick2017overcoming}  & $57.5$ & $-17.7$ & $-0.9$  & $-23.9$  & $-1.3$  & $-11.0$ \\
 & WiSE-FT~\citep{wortsman2022robust}        & $56.7$ & $-11.9$ & $+0.5$  & $-23.7$  & $-8.9$  & $-11.0$ \\
 & LoRA~\citep{hu2022lora}                   & $51.7$ & $-7.5$  & $+6.7$  & $-2.1$   & $-5.6$  & $-2.1$  \\
\cmidrule{2-8}
 & \ourmethod                                & $65.0$ & $-2.9$  & $+0.3$  & $-4.4$   & $+0.6$  & $-1.6$  \\
\midrule
\multirow{11}{*}{\rotatebox[origin=c]{90}{Llama-3-8B}}
 & Base                                      & $38.3$ & $+0.0$  & $+0.0$  & $+0.0$   & $+0.0$  & $+0.0$  \\
 & SFT                                & $55.8$ & $-10.8$ & $+0.9$  & $-23.3$  & $-5.2$  & $-9.6$  \\
 & SFT (small lr)                            & $51.7$ & $+1.7$  & $+1.9$  & $-5.2$   & $+1.1$  & $-0.1$  \\
 & FLOW~\citep{sanyal2025upweighting}        & $50.8$ & $-13.5$ & $+4.7$  & $-22.8$  & $-9.1$  & $-10.2$ \\
 & DFT~\citep{wu2025generalization}          & $49.2$ & $-8.6$  & $-0.2$  & $-18.9$  & $-11.0$ & $-9.7$  \\
 & TALR~\citep{lin2025sft}                   & $54.2$ & $-7.9$  & $-0.3$  & $-7.4$   & $-4.9$  & $-5.1$  \\
 & STM~\citep{wu2025mitigating}              & $55.0$ & $-5.3$  & $+6.9$  & $-4.0$   & $-0.8$  & $-0.8$  \\
 & L2 Reg~\citep{kirkpatrick2017overcoming}  & $60.8$ & $-14.6$ & $+1.8$  & $-20.6$  & $-3.6$  & $-9.3$  \\
 & WiSE-FT~\citep{wortsman2022robust}        & $55.8$ & $-7.7$  & $-0.9$  & $-8.1$   & $-1.3$  & $-4.5$  \\
 & LoRA~\citep{hu2022lora}                   & $53.3$ & $-5.8$  & $+6.6$  & $-2.9$   & $-6.3$  & $-2.1$  \\
\cmidrule{2-8}
 & \ourmethod                                & $56.7$ & $+1.2$  & $+2.3$  & $+1.6$   & $+1.6$  & $+1.7$  \\
\bottomrule
\end{tabular}
}
\end{table}

\begin{table}[h]
\centering
\caption{
Target-task adaptation and forgetting on Galician. Win-Tie reports the tie-adjusted head-to-head rate against the pretrained model. HellaSwag (HS), WinoGrande (WG), IFEval, and MMLU report absolute performance changes relative to the pretrained model, and Avg. $\Delta$ is their mean. Higher values indicate better preservation of general capabilities.
}
\label{tab:results_galician}
\resizebox{\textwidth}{!}{
\begin{tabular}{ll|c|cccc|c}
\toprule
\textbf{Model} & \textbf{Method} & \textbf{Win-Tie} $\uparrow$ & \textbf{HS} $\uparrow$ & \textbf{WG} $\uparrow$ & \textbf{IFEval} $\uparrow$ & \textbf{MMLU} $\uparrow$ & \textbf{Avg. $\Delta$} $\uparrow$ \\
\midrule
\multirow{11}{*}{\rotatebox[origin=c]{90}{Qwen3-4B}}
 & Base                                      & $50.0\%$  & $+0.0$  & $+0.0$  & $+0.0$   & $+0.0$  & $+0.0$  \\
 & SFT                                & $89.0\%$  & $-10.0$ & $-11.6$ & $-25.1$  & $-10.4$ & $-14.3$ \\
 & SFT (small lr)                            & $77.5\%$  & $+0.1$  & $-2.4$  & $-0.1$   & $-0.4$  & $-0.7$  \\
 & FLOW~\citep{sanyal2025upweighting}        & $90.0\%$  & $-15.5$ & $-8.7$  & $-15.6$  & $-9.4$  & $-12.3$ \\
 & DFT~\citep{wu2025generalization}          & $80.0\%$  & $+2.6$  & $-0.6$  & $+5.2$   & $+2.3$  & $+2.4$  \\
 & TALR~\citep{lin2025sft}                   & $77.5\%$  & $+0.3$  & $+0.4$  & $-3.5$   & $-0.8$  & $-0.9$  \\
 & STM~\citep{wu2025mitigating}              & $81.5\%$  & $-0.8$  & $+0.2$  & $-4.1$   & $+0.8$  & $-1.0$  \\
 & L2 Reg~\citep{kirkpatrick2017overcoming}  & $92.0\%$  & $-12.4$ & $-11.7$ & $-39.2$  & $-11.2$ & $-18.6$ \\
 & WiSE-FT~\citep{wortsman2022robust}        & $85.5\%$  & $-16.3$ & $-9.6$  & $-26.6$  & $-7.8$  & $-15.1$ \\
 & LoRA~\citep{hu2022lora}                   & $90.5\%$  & $-10.8$ & $+5.1$  & $-45.6$  & $-2.1$  & $-13.4$ \\
\cmidrule{2-8}
 & \ourmethod                                & $90.0\%$  & $-1.6$  & $-2.7$  & $+0.0$   & $-5.1$  & $-2.4$  \\
\midrule
\multirow{11}{*}{\rotatebox[origin=c]{90}{Llama-3-8B}}
 & Base                                      & $50.0\%$  & $+0.0$  & $+0.0$  & $+0.0$   & $+0.0$  & $+0.0$  \\
 & SFT                                & $97.0\%$  & $-12.2$ & $-0.8$  & $-16.7$  & $-4.2$  & $-8.5$  \\
 & SFT (small lr)                            & $94.0\%$  & $-3.6$  & $+0.2$  & $-4.1$   & $+0.4$  & $-1.8$  \\
 & FLOW~\citep{sanyal2025upweighting}        & $94.5\%$  & $-31.8$ & $-5.2$  & $-20.8$  & $-13.2$ & $-17.8$ \\
 & DFT~\citep{wu2025generalization}          & $93.5\%$  & $-5.1$  & $-3.4$  & $-20.8$  & $-5.0$  & $-8.6$  \\
 & TALR~\citep{lin2025sft}                   & $93.0\%$  & $-6.4$  & $-0.4$  & $-8.3$   & $+0.0$  & $-3.8$  \\
 & STM~\citep{wu2025mitigating}              & $95.0\%$  & $-4.2$  & $-1.4$  & $-8.3$   & $-2.0$  & $-4.0$  \\
 & L2 Reg~\citep{kirkpatrick2017overcoming}  & $97.0\%$  & $-12.2$ & $-1.8$  & $-16.7$  & $-4.0$  & $-8.7$  \\
 & WiSE-FT~\citep{wortsman2022robust}        & $96.0\%$  & $-12.6$ & $-2.0$  & $-16.7$  & $-5.0$  & $-9.1$  \\
 & LoRA~\citep{hu2022lora}                   & $97.0\%$  & $-4.6$  & $-2.2$  & $-16.7$  & $-1.0$  & $-6.1$  \\
\cmidrule{2-8}
 & \ourmethod                                & $98.5\%$  & $-1.2$  & $-0.6$  & $+0.0$   & $+0.0$  & $-0.5$  \\
\bottomrule
\end{tabular}
}
\end{table}
\paragraph{Computing the average forgetting reduction.}
We compute the reported forgetting reduction by aggregating the signed Avg.~$\Delta$ values across the six model--task settings in Figure~\ref{fig:scatter_preserve}. Let $\Delta_{m,t}$ denote Avg.~$\Delta$ for method $m$ on task setting $t$. Since more negative values indicate greater degradation, we define the total degradation as the magnitude of the signed sum:
\[
\mathrm{Deg}(m)
=
-\sum_{t=1}^{6} \Delta_{m,t}.
\]
For standard SFT, this gives
\[
\mathrm{Deg}(\mathrm{SFT}) = 61.9,
\]
whereas for \ourmethod,
\[
\mathrm{Deg}(\ourmethod) = 3.9.
\]
Thus, the relative reduction in forgetting is
\[
1 - \frac{\mathrm{Deg}(\ourmethod)}{\mathrm{Deg}(\mathrm{SFT})}
=
1 - \frac{3.9}{61.9}
\approx 93\%.
\]
\paragraph{IFEval drops.}
We observe that many fine-tuning baselines incur relatively large drops on IFEval. One possible reason is that IFEval differs from the other old-task benchmarks: unlike HellaSwag, WinoGrande, and MMLU, which are multiple-choice evaluations, IFEval requires long-form instruction following and is therefore sensitive to changes in generation behavior. In particular, degradation may shorten model responses and reduce compliance with formatting or constraint-following requirements. For example, on Qwen3-4B after KA fine-tuning with standard SFT, the selected epoch produces much shorter IFEval outputs than the pretrained model: average output length drops from $1720.6$ to $572.4$ characters, and average completion length drops from $1066.4$ to $421.8$ tokens. This suggests that part of the IFEval degradation may reflect degraded long-form instruction following rather than only loss of benchmark knowledge.
\subsubsection{Additional results for Section~\ref{sec:experiments-halluc}}

Table~\ref{tab:hallucination_science}, Figure~\ref{fig:scatter_factuality_science} and Table~\ref{tab:hallucination_galician}, Figure~\ref{fig:scatter_factuality_galician} provide the results for Science and LA respectively.

\begin{figure}[t]
\centering
\pgfplotsset{
    scatter plot science/.style={
        width=1.00\textwidth,
        height=1.00\textwidth,
        axis background/.style={fill=red!5},
        axis lines=left,
        axis line style={gray!50},
        grid=major,
        grid style={white, line width=0.8pt},
        tick style={draw=none},
        xlabel={Task Accuracy $\uparrow$},
        xlabel style={font=\small},
        ylabel style={font=\small},
        tick label style={font=\footnotesize},
        title style={font=\small\bfseries},
        xmin=38, xmax=70,
        xtick={40,45,50,55,60,65,70},
        scatter/classes={
            a={mark=*,           draw=cNormalSFT, fill=cNormalSFT},
            b={mark=square*,     draw=cFLOW,      fill=cFLOW},
            c={mark=triangle*,   draw=cDFT,       fill=cDFT},
            d={mark=diamond*,    draw=cTALR,      fill=cTALR},
            e={mark=pentagon*,   draw=cSTM,       fill=cSTM},
            f={mark=halfcircle*, draw=cL2,        fill=cL2},
            g={mark=oplus*,      draw=cWiSE,      fill=cWiSE},
            h={mark=otimes*,     draw=cLoRA,      fill=cLoRA},
            i={mark=star,        draw=cOurs,      fill=cOurs, mark size=3.5pt}
        },
    },
}
\begin{subfigure}[t]{0.32\textwidth}
\centering
\begin{tikzpicture}
\begin{axis}[scatter plot science,
    ylabel={TruthfulQA $\Delta$ $\uparrow$},
    ymin=-12, ymax=5,
]
\addplot[scatter, only marks, scatter src=explicit symbolic] coordinates {
    (62.5,  -5.9)  [a]
    (58.3,  -3.2)  [b]
    (46.7,   2.2)  [c]
    (48.3,   1.5)  [d]
    (49.2,  -9.98) [e]
    (57.5,  -6.3)  [f]
    (56.7,  -2.6)  [g]
    (51.7,  -0.6)  [h]
    (65.0,   0.2)  [i]
};
\end{axis}
\end{tikzpicture}
\caption{TruthfulQA}
\end{subfigure}\hfill
\begin{subfigure}[t]{0.32\textwidth}
\centering
\begin{tikzpicture}
\begin{axis}[scatter plot science,
    ylabel={HaluEval $\Delta$ $\uparrow$},
    ymin=-45, ymax=5,
]
\addplot[scatter, only marks, scatter src=explicit symbolic] coordinates {
    (62.5, -15.3)  [a]
    (58.3, -15.5)  [b]
    (46.7,  -7.6)  [c]
    (48.3,   0.4)  [d]
    (49.2,  -0.6)  [e]
    (57.5, -39.9)  [f]
    (56.7, -14.5)  [g]
    (51.7,  -1.6)  [h]
    (65.0,  -0.6)  [i]
};
\end{axis}
\end{tikzpicture}
\caption{HaluEval}
\end{subfigure}\hfill
\begin{subfigure}[t]{0.32\textwidth}
\centering
\begin{tikzpicture}
\begin{axis}[scatter plot science,
    ylabel={Brier $\Delta$ $\downarrow$},
    ymin=-4, ymax=8,
]
\addplot[scatter, only marks, scatter src=explicit symbolic] coordinates {
    (62.5,  1.4)  [a]
    (58.3,  -1.1)  [b]
    (46.7,  6.5)  [c]
    (48.3,  2.2)  [d]
    (49.2,  0.1)  [e]
    (57.5,  6.7)  [f]
    (56.7,  1.2)  [g]
    (51.7, 5.2)  [h]
    (65.0,  2.6)  [i]
};
\end{axis}
\end{tikzpicture}
\caption{Brier Score}
\end{subfigure}

\vspace{0.5em}
\begin{tikzpicture}
\begin{axis}[hide axis, xmin=0, xmax=1, ymin=0, ymax=1,
    legend style={at={(0.5,0.5)}, anchor=center, legend columns=5,
    font=\footnotesize, draw=none, fill=none,
    /tikz/every even column/.append style={column sep=0.5em}}]
\addlegendimage{mark=*,           only marks, draw=cNormalSFT, fill=cNormalSFT} \addlegendentry{SFT}
\addlegendimage{mark=square*,     only marks, draw=cFLOW,      fill=cFLOW}      \addlegendentry{FLOW}
\addlegendimage{mark=triangle*,   only marks, draw=cDFT,       fill=cDFT}       \addlegendentry{DFT}
\addlegendimage{mark=diamond*,    only marks, draw=cTALR,      fill=cTALR}      \addlegendentry{TALR}
\addlegendimage{mark=pentagon*,   only marks, draw=cSTM,       fill=cSTM}       \addlegendentry{STM}
\addlegendimage{mark=halfcircle*, only marks, draw=cL2,        fill=cL2}        \addlegendentry{L2 Reg}
\addlegendimage{mark=oplus*,      only marks, draw=cWiSE,      fill=cWiSE}      \addlegendentry{WiSE-FT}
\addlegendimage{mark=otimes*,     only marks, draw=cLoRA,      fill=cLoRA}      \addlegendentry{LoRA}
\addlegendimage{mark=star,        only marks, draw=cOurs,      fill=cOurs, mark size=3.5pt} \addlegendentry{\ourmethod}
\end{axis}
\end{tikzpicture}
\caption{Truthfulness, hallucination, and calibration vs.\ task accuracy on Science (Qwen3-4B).}
\label{fig:scatter_factuality_science}
\end{figure}
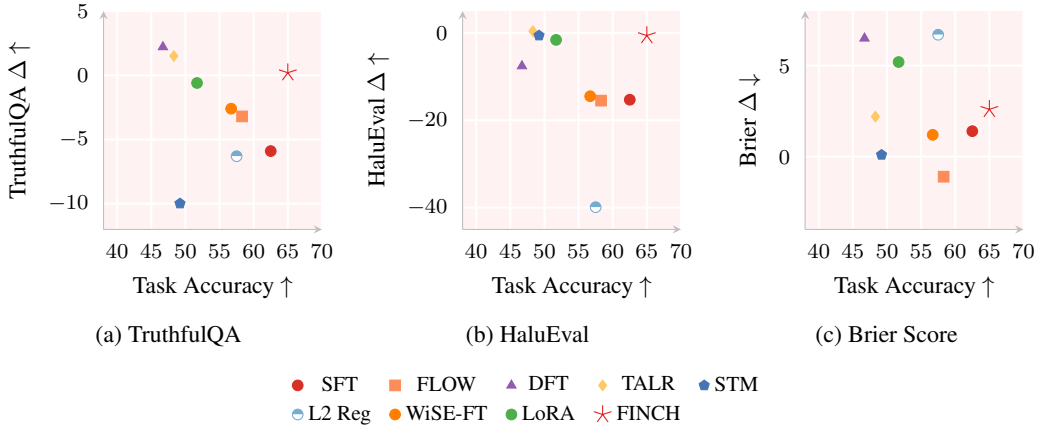

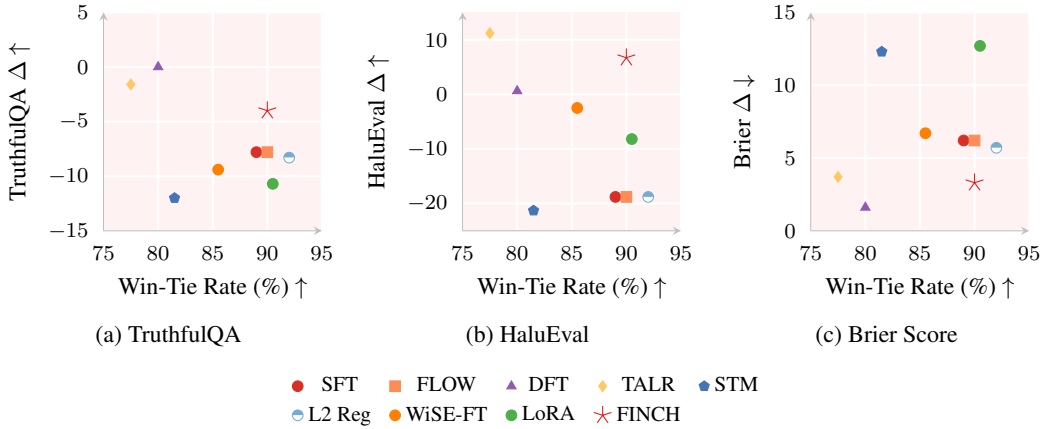
\begin{figure}[t]
\centering
\pgfplotsset{
    scatter plot gal/.style={
        width=1.00\textwidth,
        height=1.00\textwidth,
        axis background/.style={fill=red!5},
        axis lines=left,
        axis line style={gray!50},
        grid=major,
        grid style={white, line width=0.8pt},
        tick style={draw=none},
        xlabel={Win-Tie Rate (\%) $\uparrow$},
        xlabel style={font=\small},
        ylabel style={font=\small},
        tick label style={font=\footnotesize},
        title style={font=\small\bfseries},
        xmin=75, xmax=95,
        xtick={75,80,85,90,95},
        scatter/classes={
            a={mark=*,           draw=cNormalSFT, fill=cNormalSFT},
            b={mark=square*,     draw=cFLOW,      fill=cFLOW},
            c={mark=triangle*,   draw=cDFT,       fill=cDFT},
            d={mark=diamond*,    draw=cTALR,      fill=cTALR},
            e={mark=pentagon*,   draw=cSTM,       fill=cSTM},
            f={mark=halfcircle*, draw=cL2,        fill=cL2},
            g={mark=oplus*,      draw=cWiSE,      fill=cWiSE},
            h={mark=otimes*,     draw=cLoRA,      fill=cLoRA},
            i={mark=star,        draw=cOurs,      fill=cOurs, mark size=3.5pt}
        },
    },
}
\begin{subfigure}[t]{0.32\textwidth}
\centering
\begin{tikzpicture}
\begin{axis}[scatter plot gal,
    ylabel={TruthfulQA $\Delta$ $\uparrow$},
    ymin=-15, ymax=5,
]
\addplot[scatter, only marks, scatter src=explicit symbolic] coordinates {
    (89.0,  -7.8)  [a]
    (90.0,  -7.8)  [b]
    (80.0,   0.0)  [c]
    (77.5,  -1.6)  [d]
    (81.5, -12.0)  [e]
    (92.0,  -8.3)  [f]
    (85.5,  -9.4)  [g]
    (90.5, -10.7)  [h]
    (90.0,  -4.0)  [i]
};
\end{axis}
\end{tikzpicture}
\caption{TruthfulQA}
\end{subfigure}\hfill
\begin{subfigure}[t]{0.32\textwidth}
\centering
\begin{tikzpicture}
\begin{axis}[scatter plot gal,
    ylabel={HaluEval $\Delta$ $\uparrow$},
    ymin=-25, ymax=15,
]
\addplot[scatter, only marks, scatter src=explicit symbolic] coordinates {
    (89.0, -18.8)  [a]
    (90.0, -18.8)  [b]
    (80.0,   0.6)  [c]
    (77.5,  11.2)  [d]
    (81.5, -21.3)  [e]
    (92.0, -18.8)  [f]
    (85.5,  -2.5)  [g]
    (90.5,  -8.2)  [h]
    (90.0,   6.7)  [i]
};
\end{axis}
\end{tikzpicture}
\caption{HaluEval}
\end{subfigure}\hfill
\begin{subfigure}[t]{0.32\textwidth}
\centering
\begin{tikzpicture}
\begin{axis}[scatter plot gal,
    ylabel={Brier $\Delta$ $\downarrow$},
    ymin=0, ymax=15,
]
\addplot[scatter, only marks, scatter src=explicit symbolic] coordinates {
    (89.0,  6.2)  [a]
    (90.0,  6.2)  [b]
    (80.0,  1.6)  [c]
    (77.5,  3.7)  [d]
    (81.5, 12.3)  [e]
    (92.0,  5.7)  [f]
    (85.5,  6.7)  [g]
    (90.5, 12.7)  [h]
    (90.0,  3.3)  [i]
};
\end{axis}
\end{tikzpicture}
\caption{Brier Score}
\end{subfigure}

\vspace{0.5em}
\begin{tikzpicture}
\begin{axis}[hide axis, xmin=0, xmax=1, ymin=0, ymax=1,
    legend style={at={(0.5,0.5)}, anchor=center, legend columns=5,
    font=\footnotesize, draw=none, fill=none,
    /tikz/every even column/.append style={column sep=0.5em}}]
\addlegendimage{mark=*,           only marks, draw=cNormalSFT, fill=cNormalSFT} \addlegendentry{SFT}
\addlegendimage{mark=square*,     only marks, draw=cFLOW,      fill=cFLOW}      \addlegendentry{FLOW}
\addlegendimage{mark=triangle*,   only marks, draw=cDFT,       fill=cDFT}       \addlegendentry{DFT}
\addlegendimage{mark=diamond*,    only marks, draw=cTALR,      fill=cTALR}      \addlegendentry{TALR}
\addlegendimage{mark=pentagon*,   only marks, draw=cSTM,       fill=cSTM}       \addlegendentry{STM}
\addlegendimage{mark=halfcircle*, only marks, draw=cL2,        fill=cL2}        \addlegendentry{L2 Reg}
\addlegendimage{mark=oplus*,      only marks, draw=cWiSE,      fill=cWiSE}      \addlegendentry{WiSE-FT}
\addlegendimage{mark=otimes*,     only marks, draw=cLoRA,      fill=cLoRA}      \addlegendentry{LoRA}
\addlegendimage{mark=star,        only marks, draw=cOurs,      fill=cOurs, mark size=3.5pt} \addlegendentry{\ourmethod}
\end{axis}
\end{tikzpicture}
\caption{Truthfulness, hallucination, and calibration vs.\ win-tie rate on Galician (Qwen3-4B).}
\label{fig:scatter_factuality_galician}
\end{figure}
\begin{table}[h]
\centering
\caption{
Reliability analysis on Science QA. HaluEval and TruthfulQA columns show accuracy deltas relative to the pretrained model. Brier $\Delta$ is relative to the pretrained model (lower is better); TruthfulQA and HaluEval are relative to the pretrained model (higher is better).
}
\label{tab:hallucination_science}
\resizebox{1.0\textwidth}{!}{
\begin{tabular}{ll|c|cc|c}
\toprule
\textbf{Model} & \textbf{Method} & \textbf{Task Acc.} $\uparrow$ & \textbf{HaluEval $\Delta$} $\uparrow$ & \textbf{TruthfulQA $\Delta$} $\uparrow$ & \textbf{Brier $\Delta$} $\downarrow$ \\
\midrule
\multirow{10}{*}{\rotatebox[origin=c]{90}{Qwen3-4B}}
 & Base                                      & $29.2$ & $+0.0$   & $+0.0$   & $+0.0$  \\
 & SFT                                & $62.5$ & $-15.3$  & $-5.9$   & $+1.4$  \\
 & FLOW~\citep{sanyal2025upweighting}        & $58.3$ & $-15.5$  & $-3.2$   & $-1.1$  \\
 & DFT~\citep{wu2025generalization}          & $46.7$ & $-7.6$   & $+2.2$   & $+6.5$  \\
 & TALR~\citep{lin2025sft}                   & $48.3$ & $+0.4$   & $+1.5$   & $+2.2$  \\
 & STM~\citep{wu2025mitigating}              & $49.2$ & $-0.6$   & $-10.0$  & $+0.1$  \\
 & L2 Reg~\citep{kirkpatrick2017overcoming}  & $57.5$ & $-39.9$  & $-6.3$   & $+6.7$  \\
 & WiSE-FT~\citep{wortsman2022robust}        & $56.7$ & $-14.5$  & $-2.6$   & $+1.2$  \\
 & LoRA~\citep{hu2022lora}                   & $51.7$ & $-1.6$   & $-0.6$   & $+5.2$ \\
\cmidrule{2-6}
 & \ourmethod                                & $65.0$ & $-0.6$   & $+0.2$   & $+2.6$  \\
\bottomrule
\end{tabular}
}
\end{table}
\begin{table}[h]
\centering
\caption{
Reliability analysis on Galician. Win-Tie reports the tie-adjusted head-to-head rate. HaluEval and TruthfulQA columns show accuracy deltas relative to the pretrained model. Brier $\Delta$ is relative to the pretrained model (lower is better); TruthfulQA and HaluEval are relative to the pretrained model (higher is better).
}
\label{tab:hallucination_galician}
\resizebox{1.0\textwidth}{!}{
\begin{tabular}{ll|c|cc|c}
\toprule
\textbf{Model} & \textbf{Method} & \textbf{Win-Tie} $\uparrow$ & \textbf{HaluEval $\Delta$} $\uparrow$ & \textbf{TruthfulQA $\Delta$} $\uparrow$ & \textbf{Brier $\Delta$} $\downarrow$ \\
\midrule
\multirow{10}{*}{\rotatebox[origin=c]{90}{Qwen3-4B}}
 & Base                                      & $—$       & $+0.0$   & $+0.0$   & $+0.0$  \\
 & SFT                                & $89.0\%$  & $-18.8$  & $-7.8$   & $+6.2$  \\
 & FLOW~\citep{sanyal2025upweighting}        & $90.0\%$  & $-18.8$  & $-7.8$   & $+6.2$  \\
 & DFT~\citep{wu2025generalization}          & $80.0\%$  & $+0.6$   & $+0.0$   & $+1.6$  \\
 & TALR~\citep{lin2025sft}                   & $77.5\%$  & $+11.2$  & $-1.6$   & $+3.7$  \\
 & STM~\citep{wu2025mitigating}              & $81.5\%$  & $-21.3$  & $-12.0$  & $+12.3$ \\
 & L2 Reg~\citep{kirkpatrick2017overcoming}  & $92.0\%$  & $-18.8$  & $-8.3$   & $+5.7$  \\
 & WiSE-FT~\citep{wortsman2022robust}        & $85.5\%$  & $-2.5$   & $-9.4$   & $+6.7$  \\
 & LoRA~\citep{hu2022lora}                   & $90.5\%$  & $-8.2$   & $-10.7$  & $+12.7$ \\
\cmidrule{2-6}
 & \ourmethod                                & $90.0\%$  & $+6.7$   & $-4.0$   & $+3.3$  \\
\bottomrule
\end{tabular}
}
\end{table}
\newpage
\section{Extended Related Work}
\label{app:extended-related-work}
\subsection{Catastrophic Forgetting}
Catastrophic forgetting refers to the degradation of previously acquired knowledge when a model is trained on new data~\citep{mccloskey1989catastrophic, french1999catastrophic}. We organize existing mitigation strategies into three groups.

\paragraph{Methods with Access to Pretraining Data}
Replay mixes old training data into fine-tuning and is the most effective mitigation strategy~\citep{rolnick2019experience, de2019episodic, scialom2022fine}, but pretraining data is rarely available for modern LLMs~\citep{grattafiori2024llama, yang2025qwen3}. Some replay methods also use synthetic data when real old data is unavailable~\citep{huang2024mitigating, chen2024continual, bansal2025context}. However, synthetic replay introduces additional generation and filtering cost, depends on the quality and coverage of the generated data, and changes the fine-tuning distribution. Other methods constrain updates to the orthogonal subspace of old task representations~\citep{lin2022trgp, singh2025mitigating} or identify and protect important parameters~\citep{song2025alleviate}, but similarly require old data and are evaluated only on small image benchmarks such as MNIST~\citep{lecun2002gradient} and CIFAR-10~\citep{krizhevsky2009learning}. Like all data-oblivious methods, \ourmethod\ requires no access to pretraining data.

\paragraph{Data-Oblivious Methods}
Kirkpatrick et al.~\cite{kirkpatrick2017overcoming} regularize weights to stay close to their pretrained values. LoRA~\citep{hu2022lora} constrains updates to a low-rank subspace, reducing forgetting but hurting target-domain performance~\citep{biderman2024lora}. Distillation-based methods~\citep{agarwal2024policy, lu2025onpolicydistillation} replace SFT with student rollouts scored by a stronger teacher; this is computationally expensive and struggles when the base model assigns low probability to relevant sequences~\citep{shao2024deepseekmath}. A separate line of work reduces forgetting by downweighting high-loss tokens or sequences: Sanyal et al.~\cite{sanyal2025upweighting} upweight low-loss sequences, Wu et al.~\cite{wu2025mitigating} mask tokens above a loss threshold, Lin et al.~\cite{lin2025sft} scale each token's loss by $p^{1/\tau}$ where $p$ is the token probability and $\tau$ is a constant, and Wu et al.~\cite{wu2025generalization} rescale gradients by token probability. However, for many tasks learning hard tokens is essential, so suppressing them hurts target performance. \ourmethod\ leaves the training objective unchanged and controls forgetting solely through the learning rate.

\paragraph{Theoretical Analysis}
Forgetting has been analyzed theoretically for linear regression~\citep{evron2022catastrophic}, two-layer CNNs~\citep{li2025towards}, and linear attention~\citep{lee2026fine}, but these results are specific to the simple models considered and do not transfer to general LLMs. Li et al.~\citep{li2024revisiting} propose a sharpness-based optimization method to mitigate forgetting. FINCH instead keeps the standard SFT objective and AdamW setup fixed, and controls forgetting through a loss-adaptive learning-rate schedule.

\subsection{Learning Rate}
The learning rate determines both the speed and outcome of training, with larger rates tending to find wider, better-generalizing minima~\citep{smith2019super} but risking divergence beyond a critical threshold~\citep{lewkowycz2020large}. In practice, schedules typically combine linear warmup~\citep{goyal2017accurate} with cosine decay~\citep{loshchilov2016sgdr}, and this pattern is especially common for transformers~\citep{vaswani2017attention}. The role of warmup has been studied from two complementary angles. Kalra et al.~\cite{kalra2024warmup} argue that small initial learning rates allow the model to move into flatter, better-conditioned regions of the loss landscape before the rate is increased. Kosson et al.~\cite{kosson2024analyzing} attribute early instability to three compounding factors (bias correction inflation in Adam, large angular parameter updates, and high gradient signal-to-noise ratio) and show that warmup implicitly controls all three.

In the context of catastrophic forgetting, Kenneweg et al.~\cite{kenneweg2022intelligent} assign a separate learning rate to each layer in BERT and tune all rates jointly via Bayesian optimization. Scaling this to modern LLMs is infeasible: a 36-layer model such as Qwen3 requires 25 to 30 hyperparameters, the number of required trials scales with dimensionality~\citep{kandasamy2015high}, and each trial requires a full fine-tuning run. Lin et al.~\cite{lin2025sft} prescribe a fixed small learning rate and provide a theoretical justification. They assume each gradient step changes the model by at most $\kappa$ in KL divergence, and under this assumption approximate the updated model via an exponential tilt; they then bound the total change in old-task loss. The KL constraint is assumed rather than derived, and the exponential tilt carries approximation carries an $O(\kappa)$ error per step, accumulating to $O(T\kappa)$ over $T$ steps. Since a smaller learning rate reduces $\kappa$ but increases $T$, the bound does not improve in the regime they advocate. Our approach derives the learning rate for which the KL constraint holds, and our bound applies without the growing error (Theorem~\ref{thm:stepwise-bound}). As shown in Section~\ref{sec:experiments-preserves}, a fixed small learning rate also underperforms empirically, failing to reach competitive target-task accuracy.

Trust-region methods, generally developed for reinforcement learning, also limit update size during training~\citep{schulman2015trust, schulman2017proximal, zhu2025proximal}. However, these methods are not designed specifically for catastrophic-forgetting mitigation and modify or explicitly constrain the training objective. Moreover, probability-ratio constraints can be conservative for low-probability tokens, which is undesirable in our setting where rare vocabulary, new facts, and low-resource-language tokens must still be learned~\cite{qi2026rethinking}. FINCH instead keeps the SFT objective unchanged and studies whether the learning-rate schedule alone can control forgetting.